\documentclass{article}

\usepackage{microtype}
\usepackage{graphicx}
\usepackage{subcaption}
\usepackage{booktabs} % for professional tables
\usepackage{subcaption}
\usepackage{hyperref}

\usepackage[preprint]{configuration/icml2026}

\usepackage{amsmath}
\usepackage{amssymb}
\usepackage{mathtools}
\usepackage{amsthm}
\usepackage{multirow}
\usepackage{graphicx}
\usepackage{tcolorbox}
\usepackage{amsmath}

\usepackage[capitalize,noabbrev]{cleveref}

\usepackage[textsize=tiny]{todonotes}
\usepackage{amsmath,amssymb,amsthm}
\newtheorem{theorem}{Theorem}[section]
\newtheorem{lemma}[theorem]{Lemma}

\newtheorem{proposition}[theorem]{Proposition}
\newtheorem{remark}[theorem]{Remark}

\providecommand{\mpi}{\boldsymbol{\pi}}

\providecommand{\cA}{\mathcal{A}}

\providecommand{\cR}{\mathcal{R}}
\providecommand{\cS}{\mathcal{S}}

\providecommand{\cV}{\mathcal{V}}

\newenvironment{talign*}
{\csname align*\endcsname}
{\endalign}

\usepackage[utf8]{inputenc}         %
\usepackage[T1]{fontenc}            %
\usepackage{url}                    %
\usepackage{booktabs}               %
\usepackage{amsfonts}               %
\usepackage{nicefrac}               %
\usepackage{microtype}              %
\usepackage{xcolor}                 %
\usepackage{algorithm}
\usepackage{algorithmic}
\usepackage{graphicx}
\usepackage{subcaption}
\usepackage[flushleft]{threeparttable}
\usepackage{float}
\usepackage{multirow}
\usepackage{xspace}
\usepackage{natbib}
\usepackage{enumitem}
\usepackage[font=small]{caption}
\usepackage{autobreak}
\usepackage{sidecap}
\usepackage{wrapfig}
\usepackage{bbding}
\usepackage[toc, page, header]{appendix}
\usepackage{tikz}
\usepackage{xcolor}
\usepackage{pifont}
\usepackage{mdframed}

\hypersetup{
  colorlinks=true,
  linkcolor=blue,
  citecolor=blue,
  urlcolor=blue
}

\definecolor{coral}{RGB}{255,127,80}
\definecolor{darkgreen}{RGB}{0,100,0}
\definecolor{darkyellow}{RGB}{204,153,0}
\definecolor{salmon}{RGB}{250,128,114}

\usepackage{xspace}
\newcommand{\method}{\textsc{LIE}\xspace}

\usepackage{xurl}

\icmltitlerunning{Learn to Explore In-Context via Length-Incentivized Reinforcement Learning}

\begin{document}

\twocolumn[
  \icmltitle{Think Longer to Explore Deeper: Learn to Explore In-Context via Length-Incentivized Reinforcement Learning}

  % It is OKAY to include author information, even for blind submissions: the
  % style file will automatically remove it for you unless you've provided
  % the [accepted] option to the icml2026 package.

  % List of affiliations: The first argument should be a (short) identifier you
  % will use later to specify author affiliations Academic affiliations
  % should list Department, University, City, Region, Country Industry
  % affiliations should list Company, City, Region, Country

  % You can specify symbols, otherwise they are numbered in order. Ideally, you
  % should not use this facility. Affiliations will be numbered in order of
  % appearance and this is the preferred way.
  \icmlsetsymbol{equal}{*}
  \icmlsetsymbol{letter}{$\dagger$}
  \begin{icmlauthorlist}
    \icmlauthor{Futing Wang}{equal,zju,westlake,shailab}
    \icmlauthor{Jianhao Yan}{equal,zju,westlake,shailab}
    \icmlauthor{Yun Luo}{shailab,letter}
    \icmlauthor{Ganqu Cui}{shailab}
    \icmlauthor{Zhi Wang}{nju,shailab}
    \icmlauthor{Xiaoye Qu}{shailab}
    \icmlauthor{Yue Zhang}{westlake,wias}
    \icmlauthor{Yu Cheng}{cuhk,letter}
    \icmlauthor{Tao Lin}{westlake,letter}
  \end{icmlauthorlist}

  \icmlaffiliation{zju}{Zhejiang University}
  \icmlaffiliation{westlake}{Westlake University}
  \icmlaffiliation{shailab}{Shanghai AI Laboratory}
  \icmlaffiliation{nju}{Nanjing University}
  \icmlaffiliation{wias}{Institute of Advanced Technology, Westlake Institute for Advanced Study}
  \icmlaffiliation{cuhk}{The Chinese University of Hong Kong}

  \icmlcorrespondingauthor{Yun Luo}{luoyun1@pjlab.org.cn}
  \icmlcorrespondingauthor{Yu Cheng}{chengyu@cse.cuhk.edu.hk}
  \icmlcorrespondingauthor{Tao Lin}{lintao@westlake.edu.cn}

  % You may provide any keywords that you find helpful for describing your
  % paper; these are used to populate the "keywords" metadata in the PDF but
  % will not be shown in the document
  \icmlkeywords{Machine Learning, ICML}

  \vskip 0.3in
]

% this must go after the closing bracket ] following \twocolumn[ ...

% This command actually creates the footnote in the first column listing the
% affiliations and the copyright notice. The command takes one argument, which
% is text to display at the start of the footnote. The \icmlEqualContribution
% command is standard text for equal contribution. Remove it (just {}) if you
% do not need this facility.

% Use ONE of the following lines. DO NOT remove the command.
% If you have no special notice, KEEP empty braces:
\printAffiliationsAndNotice{}  % no special notice (required even if empty)
% Or, if applicable, use the standard equal contribution text:
% \printAffiliationsAndNotice{\icmlEqualContribution}

\begin{abstract}
  Achieving effective test-time scaling requires models to engage in In-Context Exploration --- the intrinsic ability to generate, verify, and refine multiple reasoning hypotheses within a single continuous context.
  Grounded in State Coverage theory, our analysis identifies a critical bottleneck to enabling this capability: while broader state coverage requires longer reasoning trajectories, the probability of sampling such sequences decays exponentially during autoregressive generation, a phenomenon we term the ``Shallow Exploration Trap''.
  To bridge this gap, we propose Length-Incentivized Exploration(\method).
  This simple yet effective recipe explicitly encourages models to explore more via a length-based reward coupled with a redundancy penalty, thereby maximizing state coverage in two-step manner.
  Comprehensive experiments across different models (Qwen3, Llama) demonstrate that \method effectively incentivize in-context exploration.
  As a result, our method achieves an average improvement of 4.4\% on in-domain tasks and a 2.7\% gain on out-of-domain benchmarks.
  Our code is publicly available in \url{https://github.com/LINs-lab/LIE}.
\end{abstract}
% required
\section{Introduction}

Scaling test-time computation, often conceptualized as enabling models to ``think harder'' before answering, has emerged as a powerful paradigm for breaking the performance ceiling of Large Language Models (LLMs) \citep{snell2024scaling, wu2024inference, liu2025can}.
Broadly, Test-Time Scaling (TTS) strategies fall into two primary regimes: Parallel Scaling~\citep{lightman2023let, snell2024scaling, liu2025can, brown2024large,wang2022self}, which aggregates outputs from multiple independent samples, and Sequential Scaling~\citep{guo2025deepseek, jaech2024openai,kumar2024training}, which prioritizes extended reasoning chains or iterative refinement.
While Parallel Scaling structures the search space externally,
Sequential Scaling via Long CoT represents an intrinsic dimension of TTS, as it directly harnesses the model's intrinsic reasoning capabilities \cite{snell2024scaling}.
% This internal dimension gives rise to in-context exploration, a capability that complements external aggregation by enhancing the depth and width of individual trajectories \citep{pan2025learning, venkatraman2025recursive}.

\looseness=-1
\begin{figure}
  \centering
  \includegraphics[width=\linewidth]{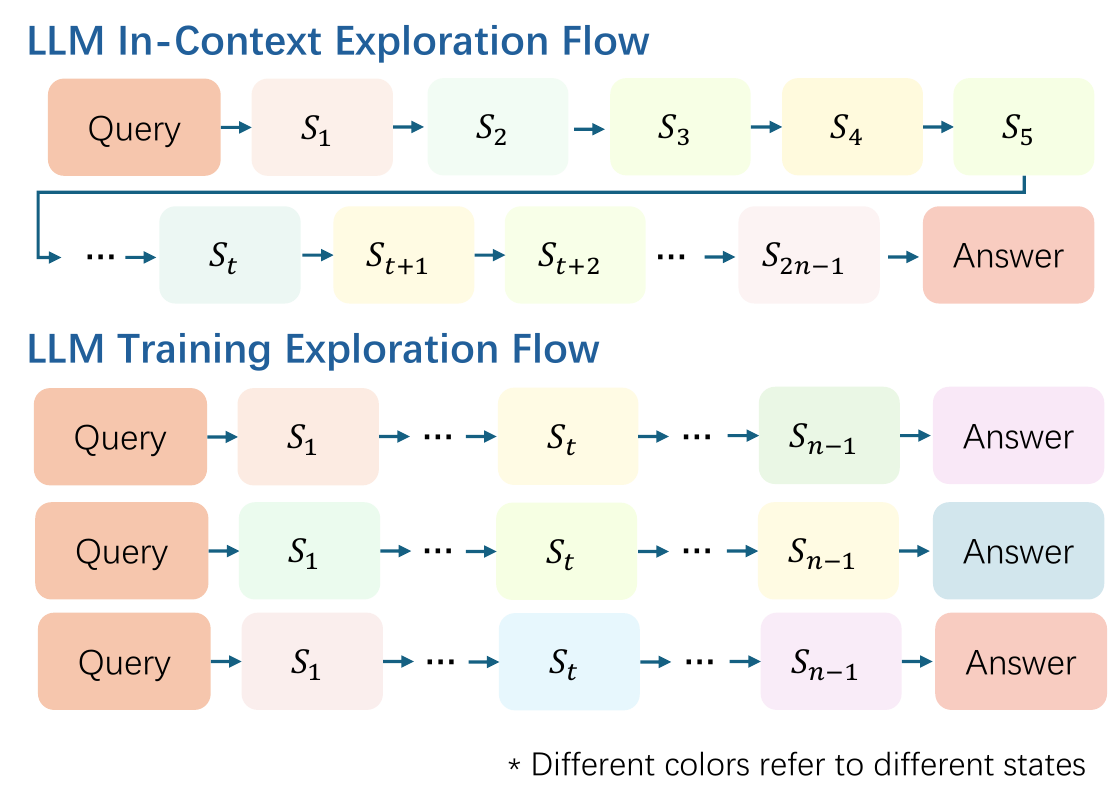}
  \vspace{-1.5em}
  \caption{\textbf{The difference between In-Context Exploration and Training Exploration.}
    Our framework distinguishes between the exploration of the training process and in-context inference.
    In the training phase, reinforcement learning incentivizes the model to explore and learn from diverse state distributions.
    In contrast, during test-time inference, in-context exploration empowers the model to actively traverse and navigate states.}
  \label{fig:intro}
  \vspace{-2em}
\end{figure}

% We refer to this intrinsic reasoning capability as \textit{in-context exploration}, which complements external aggregation \citep{pan2025learning, venkatraman2025recursive} by enhancing the depth and breadth of individual trajectories through the generation, verification, and refinement of hypotheses within a continuous context \citep{gandhi2024stream, setlur2025e3}.
One iconic capability among this intrinsic reasoning is \textit{in-context exploration}~\cite{setlur2025e3}, which is achieved through the generation, verification, and refinement of multiple hypotheses within a continuous context \citep{gandhi2024stream, setlur2025e3}.
% In this way, enabling the LLM to explore in-context is essential to activate the sequential scaling.
Consequently, fostering such in-context exploration serves as the critical catalyst for unlocking the full potential of sequential scaling.
% It complements external aggregation by enhancing the depth and breadth of individual trajectories \citep{pan2025learning, venkatraman2025recursive}.
% This is achieved through the generation, verification, and refinement of hypotheses within a continuous context \citep{gandhi2024stream, setlur2025e3}.
% \textit{In-Context Exploration}—the capability to generate, verify, and refine hypotheses within a continuous context—is central to internalizing the search process via autoregressive refinement \citep{gandhi2024stream}.
% \tao{Training Exploration is not formally defined. does it correspond to the conventional training?}
% As illustrated in Figure~\ref{fig:intro}, 
% % unlike exploration in training which focuses on learning from diverse state distributions, 
% in-context exploration empowers the model to sequentially explore diverse states such as different hypotheses, verification, and refinement, and finally induce the final answer.
As illustrated in Figure~\ref{fig:intro}, in-context exploration enables the model to sequentially traverse diverse reasoning states
% , encompassing hypothesis generation, self-verification, and iterative refinement, 
to find the correct answer.
% With 
However,
% while Reinforcement Learning (RL) recipes have demonstrated potential in scaling test-time compute --- allowing the model to improve performance by “thinking” for longer \citep{guo2025deepseek, deepscaler2025, setlur2025e3} --- 
how to effectively train LLMs to acquire in-context exploration remains underexplored.

To study this problem, we first analyze the bottleneck of in-context exploration through the lens of Count-Based Exploration theory in RL training \citep{auer2002using}, which establishes state coverage as a fundamental proxy for exploration quality.
We extend the theory into test-time in-context exploration and adopt state coverage, the in-context state diversity, as a theoretical proxy for the exploration quality.
Figure ~\ref{fig:intro} demonstrates the difference between in-context exploration and the classic rl training exploration.
Our theoretical analysis of state coverage exposes a critical bottleneck (Remark ~\ref{remark:conflit}): \textbf{achieving broader state coverage requires longer reasoning sequences, yet such sequences face exponentially diminishing sampling probabilities (\textit{Shallow Exploration Trap})} (Lemma ~\ref{lemma:shallow_exploration_trap}).
Empirically, we analyze the training dynamics of GSPO and GRPO baselines.
Our observations validate our theoretical results and also reveal that these methods implicitly incentivize response length to some degree, while decreasing states density (Section \ref{sec:analysis_baseline}).
This observation prompts our primary research question:
\begin{center}
  \textit{How to effectively incentivize in-context exploration in RL for test-time scaling?}
\end{center}
% \textit{Can we explicitly encourage models to think longer as a mechanism to activate broader in-context exploration, thereby achieving more effective test-time scaling?}

To address this, we propose Length-Incentivized Exploration (\method), an RL recipe that maximizes in-context state coverage, in a two-step manner.
The length reward elevates the upper bound of states, and the redundancy penalty matches the number of states given length~(as introduced in Section ~\ref{sec:recipe}).
% employs length rewards to facilitate exploration within vast trajectory spaces and enhance in-context exploration.
% Furthermore, to mitigate repetition collapse and ensure training stability, we incorporate a tailored repetition penalty into the reward function 
\looseness=-1
Experiments demonstrate that \method significantly improves performance compared to GRPO and GSPO baselines across diverse models (Qwen3~\citep{yang2025qwen3} and LLaMA~\citep{meta2024llama}).
Specifically, our method achieves consistent performance boosts, yielding a 4.4\% average gain on in-domain reasoning tasks and a 2.7\% improvement on out-of-domain benchmarks, underscoring its superior generalization capabilities.
% Our results reveal that explicitly encouraging long-trajectory reasoning---even within a limited 8k training window---effectively activates the model's intrinsic in-context exploration capacity, leading to superior performance when evaluating on a 32k generation scale.
% As trajectory length increases, the model's ability to visit ``deep'' reasoning states grows, which is strongly correlated with solving complex, multi-step problems. \futing{evidence?}

Our key contributions are summarized as follows:
\begin{itemize}[nosep, leftmargin=12pt]
  \item \textbf{Identifying the In-Context Exploration Bottleneck.}. We identify ``Shallow Exploration Trap'' as the bottleneck of in-context exploration, substantiating its existence through both theoretical derivations (Proposition~\ref{propos:length_as_capacity} \& Lemma~\ref{lemma:shallow_exploration_trap}) and empirical validations (Figure~\ref{fig:baseline}).
  \item \textbf{Length-Incentivized Exploration (\method) Recipe.} We propose a simple yet effective training recipe that explicitly incentivizes the model to overcome the bottleneck, enabling the model to explore during both training and inference. \looseness=-1
  \item \textbf{Empirical Validation and Test-Time Scaling}.
        Comprehensive experiments demonstrate that \method significantly outperforms standard baselines, achieving effective test-time scaling and eliciting more cognitive behaviors.
\end{itemize}

\section{Related Work}
% \paragraph{Test-time scaling}
\paragraph{Scaling Test-Time Compute via Long CoT.}
% what is scaling thinking and signicance
Recent works have shown remarkable potential in scaling test-time computation \citep{snell2024scaling} by training models to generate the long chains of thoughts (CoT) that enable strategic behaviors, including verification, self-correction \citep{kumar2024training}, etc.
Scaling test-time computation significantly enhances the performance of Large Language Models \citep{guo2025deepseek,team2025kimi,deepscaler2025,jaech2024openai} on various domains.
% how to scaling
% RL & s1
Reinforcement Learning (RL), in particular, facilitates ``natural'' scaling of test-time computation through intrinsic adaptive exploration and strategy application \citep{guo2025deepseek,team2025kimi, gandhi2025cognitive, setlur2025e3, yeo2025demystifying}.
These processes are mechanistically driven by negative gradients \citep{setlur2025e3,zhu2025surprising} or high entropy \citep{cui2025entropy,wang2025beyond,cheng2025reasoning}.
S1~\citep{muennighoff2025s1} forces the model to think longer through explicitly forcing the generation of additional tokens to extrapolate thinking.
However, we focus directly on increasing the computation usage during reinforcement training to activate in-context exploration.
\vspace{-1em}
\paragraph{Length-Aware Reasoning.}
Recent studies have investigated length-aware reasoning, specifically aiming to mitigate overthinking and improve both efficiency and controllability.
Several works \citep{jiang2025think,huang2025adactrl,zhang2025adaptthink,kang2025c3ot,zhang2025othink,yu2025think,yang2025towards,liu2025learnreasonefficientlyadaptive} propose adaptively switching between long and short Chain-of-Thought (CoT) based on problem difficulty through training.
Others \citep{zhang2025alphaone,ma2025cot} focus on controlling token usage during inference, while some studies address reasoning within a specific budget \citep{wen2025budgetthinker,aggarwal2025l1, xu2025scalable}.
In contrast to these efficiency-centric approaches that primarily aim to optimize token usage or curb overthinking, we posit that extended reasoning is a prerequisite for reaching deep reasoning states.
We investigate explicitly scaling up reasoning trajectories not as a burden, but as a mechanism to break the ``Shallow Exploration Trap'' and activate intrinsic exploratory behaviors.

\section{Background} \label{sec:background}
% \tao{use bold for vector/matrix/tensor, e.g., $\yy$}
In this section, we formulate LLM reasoning as a Markov Decision Process (MDP) and review the theoretical foundations of Count-Based Exploration in traditional reinforcement learning.

\subsection{MDP Formulation of LLM Reasoning}
\paragraph{LLM Reasoning as an MDP.}
We model the autoregressive generation process of a Large Language Model (LLM) as a deterministic MDP tuple $(\cS, \cA, \mpi, T)$, where
\begin{itemize}[nosep, leftmargin=12pt]
  \item \textbf{State Space ($\cS$):}
        The state space $\cS$ consists of all possible sequences of tokens from the vocabulary $\cV$.
        Crucially, a specific state $s_t \in \cS$ at time step $t$ is the concatenation of the input query (prompt) $x$ and the thought chain generated so far $y_{<t}$. Formally, $s_t = [x, y_1, \dots, y_{t-1}]$.

  \item \textbf{Action Space ($\cA$):}
        The action space is discrete and equivalent to the model's vocabulary $\cV$, where an action $a_t \in \cV$ corresponds to selecting the next token to append to the current sequence.

  \item \textbf{Transition Dynamics ($T$):}
        The transition is deterministic. Given the current state $s_t$ and action $a_t$, the next state is uniquely determined by appending the token to the history: $s_{t+1} = s_t \oplus a_t$. The process terminates when a special end-of-sequence (EOS) token is generated.

  \item \textbf{Policy ($\mpi_\theta$):}
        The LLM functions as a parameterized policy $\mpi_\theta(a_t | s_t)$, which maps the current context $s_t$ to a probability distribution over the vocabulary $\cV$.
\end{itemize}
Due to the combinatorial nature of language, the size of $\cS$ grows exponentially with sequence length ($|\cS| \approx |\cV|^L$), making the state space extremely vast and sparse.

\vspace{-1em}
\paragraph{Reinforcement Learning for LLM Reasoning.}
To optimize the policy $\mpi_\theta$ to maximize the expected return $J(\pi_\theta) = \mathbb{E}[\sum r_t]$, the optimization is typically performed via \textbf{Policy Gradient} \citep{sutton1999policy}:
\[
  \nabla_\theta J(\mpi_\theta) = \mathbb{E}_{\tau}\left[\sum_{t=0}^{T} \nabla_\theta \log \mpi_\theta(a_t \mid s_t) \, A^ {\mpi}(s_t, a_t)\right] \,,
\]
where $A^{\mpi}$ represents the advantage of taking action $a_t$.
GRPO and GSPO are instantiations of this framework.

\textbf{Group Relative Policy Optimization (GRPO).}
GRPO~\citep{shao2024deepseekmath} performs optimization at the token level without a value function.
It utilizes the standard per-token probability ratio:
\begin{equation}
  \rho_{i,t}(\theta) = \frac{\mpi_\theta(y_{i,t}|x, y_{i,<t})}{\mpi_{\theta_{\text{old}}}(y_{i,t}|x, y_{i,<t})} \,,
\end{equation}
where $y_{i,t}$ is the $t$-th token of the $i$-th sequence. The advantage is computed via group normalization:
\begin{equation}
  \hat{A}_i = \frac{R_i - \text{mean}(\{R_j\}_{j=1}^G)}{\text{std}(\{R_j\}_{j=1}^G)} \,.
\end{equation}
The objective averages the PPO-clip loss \citep{schulman2017proximal} over all tokens in the generated sequences:
\begin{equation}
  \small
  \begin{split}
    J^{\text{GRPO}}(\theta) & = \mathbb{E}_{x, \{y_i\} \sim \mpi_{\theta_{\text{old}}}} \bigg[ \frac{1}{G} \sum_{i=1}^G \frac{1}{|y_i|} \sum_{t=1}^{|y_i|} \min \Big( \rho_{i,t}(\theta)\hat{A}_i, \\
                            & \quad \text{clip}(\rho_{i,t}(\theta), 1-\epsilon, 1+\epsilon)\hat{A}_i \Big) \bigg] \,.
  \end{split}
\end{equation}

\textbf{Group Sequence Policy Optimization (GSPO).}
GSPO~\citep{zheng2025group} elevates optimization to the sequence level using length-normalized importance ratios:
\begin{equation}
  \begin{split}
    \rho_i(\theta) = \left( \prod_{t=1}^{|y_i|} \frac{\mpi_\theta(y_{i,t}|x, y_{i,<t})}{\mpi_{\theta_{\text{old}}}(y_{i,t}|x, y_{i,<t})} \right)^{1/|y_i|} \,,
  \end{split}
\end{equation}
where the $1/|y_i|$ exponent reduces variance from varying sequence lengths.
Using the same group-based advantage $\hat{A}_i$ as GRPO, the objective is computed once per sequence.
\subsection{Theoretical Foundation: Incentivizing State Coverage via Count-Based Exploration} \label{sec:count-based exploration}
Count-based exploration is a fundamental strategy to address the exploration-exploitation tradeoff.
Central to this approach is the \textbf{state visitation count} $N(s)$, which records the cumulative frequency of visiting a state $s$.
These counts serve as a proxy for \textbf{state coverage} --- the diversity of states visited during exploration.
Guided by the principle of \textit{Optimism in the Face of Uncertainty} \cite{auer2002finite}, count-based methods augment the extrinsic reward with an exploration bonus $b(s)$, typically inversely proportional to the counts (e.g., $\propto 1/\sqrt{N(s)}$).
% \looseness=-1
% \paragraph{The Upper-Confidence Bound Framework.}

\vspace{-1em}
\paragraph{Theoretical Guarantees for Exploration.}
The optimality of count-based exploration is formally established in the Multi-Armed Bandit (MAB) setting. More detailed descriptions are shown in Appendix \ref{appedix: count_based_exploration}.
The Upper-Confidence Bound (UCB) algorithm selects action $a_t$ by balancing reward estimates with visitation counts:
\begin{equation}\label{eq:ucb}
  a_t = \mathop{\arg\max}_{a\in\mathcal{A}} \hat{R}_t(a) + \sqrt{\frac{2\log t}{n_t(a)}} \,,
\end{equation}
where $\hat{R}_t(a)$ and $n_t(a)$ is the reward estimate and visitation count for action $a$ at time $t$.
The bonus term $\sqrt{2\log t / n_t(a)}$ decreases as the action is sampled more frequently.
Crucially, this bonus minimizes cumulative regret, yielding a theoretical guarantee for efficiently identifying the optimal action, formalized as:
% \tao{be careful with the colored text in appendix.}
\begin{theorem}[\textbf{Optimality of Count-based Exploration}~\citep{auer2002using}] \label{theo:optimality_count_based_exploration}
  In an MAB setting, let $L(T)=\mathbb{E}[\sum_{t=1}^T\left(R^*-R(a_t)\right)]$ denote the total regret over $T$ steps, where $R(a)=\mathbb{E}_{\mathcal{R}^a}[R]$ is the expected reward for any action $a$.
  $\cR^a$ is the reward distribution for action $a$ and $R^*$ is the reward for the optimal action.
  The UCB algorithm (Eq.~\ref{eq:ucb}) achieves this optimal bound:
  \begin{equation}
    \lim_{T\to\infty}L(T) \le 8\log T \cdot\sum_{a|\Delta_a>0}\Delta_a \,,
  \end{equation}
  where $\Delta_a$ is the reward gap between action $a$ and the optimal action.
\end{theorem}
% This principle extends to MDPs by counting state-action pairs $n(s,a)$, providing near-optimal guarantees within the PAC-MDP framework~\citep{strehl2008analysis, bellemare2016unifying}. 

\begin{remark} \label{remark:theo_insight}
  The count-based exploration principle provides a fundamental insight for LLM reasoning: \textbf{effective exploration requires maximizing state coverage}.
  % In standard RL training, visitation counts are aggregated across the entire state space over many episodes.
  % For in-context exploration, however, we must adapt this principle to quantify state diversity within a \emph{single} reasoning trajectory.
  It is crucial to note that in standard RL training, the visitation count $N(s)$ is aggregated across the \textbf{whole state space} over many training episodes.
  For in-context exploration, however, we must adapt this principle to quantify state diversity within a \emph{single} reasoning trajectory.
\end{remark}

\section{In-Context Exploration} \label{sec:problem}
While the standard MDP formulation addresses training-time optimization, \textit{In-Context Exploration} addresses how the model scales at test time.
As highlighted in Remark~\ref{remark:theo_insight} and Figure~\ref{fig:intro}, the focus shifts from global state coverage to maximizing state coverage within a single trajectory.
% Specifically, within a reasoning trajectory $\boldsymbol{\tau}$, to facilitate in-context exploration, we encourage the model to explore more states during test time.
% The key distinction lies in the scope: training exploration aggregates visitation counts across the entire state space over many episodes, whereas in-context exploration quantifies exploration diversity within one continuous reasoning chain.
\looseness=-1

\subsection{In-Context State Space}
\label{sec:in_context_state}
% Inspired by the count-based exploration theory in Section~\ref{sec:count-based exploration}, we aim to quantify exploration within a single reasoning trajectory $\boldsymbol{\tau} = (y_1, y_2, \dots, y_L)$, where $L$ is the trajectory length.
\looseness=-1
% We first define the in-context states $S_{\text{IC}}(\tau) = \{s_1, s_2, \dots\}$ as the sequence of states visited during generation.
% However, a direct application of count-based exploration faces the state uniqueness problem: since every state $s_t = [x, y_{<t}]$ encapsulates the entire unique history, no raw state is ever visited twice (i.e., the visitation count $N(s_t) \equiv 1$).
% This makes measurement of state coverage trivial and ineffective.
\paragraph{Defining in-context states.}
Given a single reasoning trajectory $\boldsymbol{\tau} = (y_1, y_2, \dots, y_L)$, where $L$ is the trajectory length, we define the in-context state space as the sequence of autoregressive states visited during generation:
\begin{equation} \label{eq:in_context_states}
  S_{\text{IC}}(\boldsymbol{\tau}) = \{s_1, s_2, \dots, s_L\} \,,
\end{equation}
A direct application of count-based states to in-context states faces a fundamental challenge: since every state $s_t$ contains a unique history prefix, no raw state is ever visited twice within a trajectory, rendering state coverage measurement trivial and meaningless.
\looseness=-1
\paragraph{State abstraction.}
To derive meaningful visitation counts, we introduce a state abstraction function $\phi$ that maps raw states to abstract states capturing the \textit{logical state} of reasoning rather than the literal history.
Following evidence that semantics is predominantly encoded in immediate local patterns~\citep{levy2014neural}, we employ \textit{last-$n$-grams} as our abstraction:
\begin{equation}\label{eq:state_abstraction}
  \phi(s_t) = (y_{t-n+1}, \dots, y_t) \,.
\end{equation}
By aggregating states based on these local patterns, we obtain pseudo-visitation counts that meaningfully estimate state coverage within a single context.
\looseness=-1

\begin{figure}[t]
  \centering
  \centering
  \includegraphics[width=\linewidth]{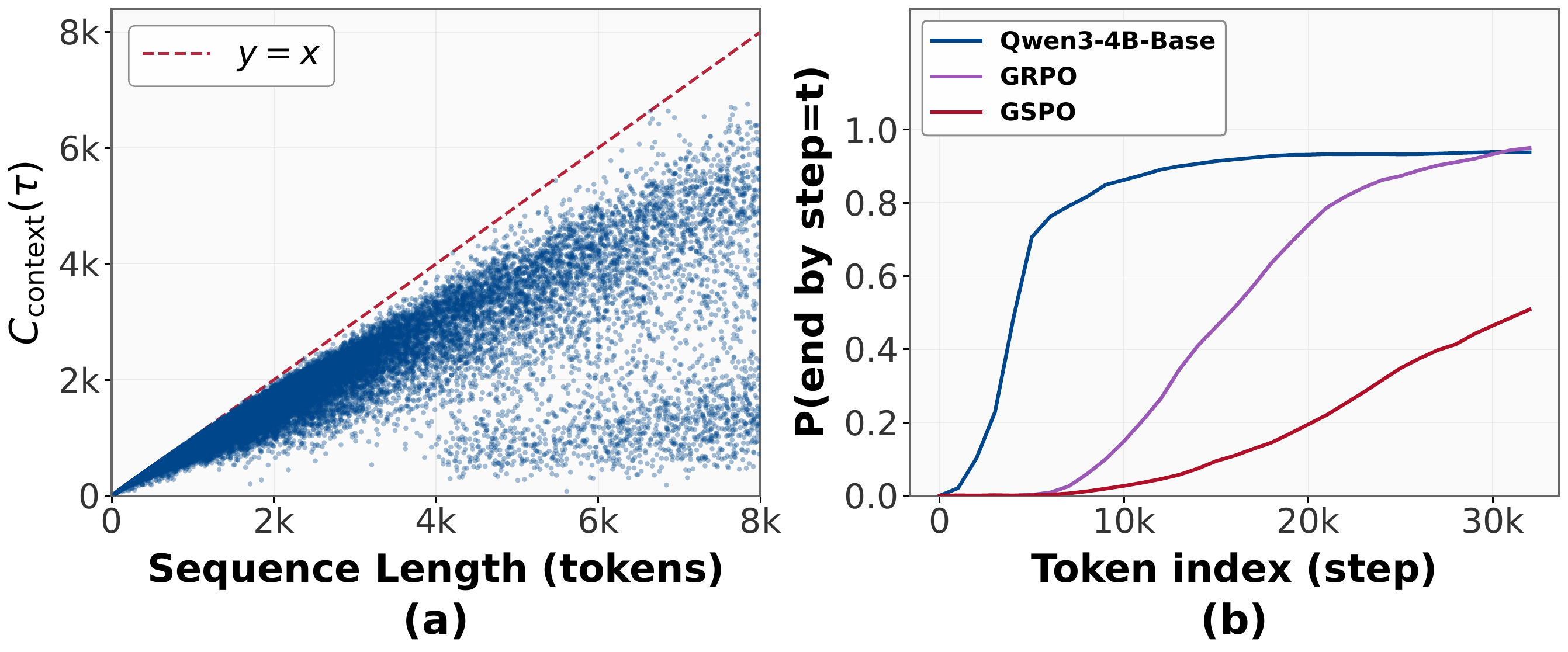}
  \caption{\textbf{The Length Bottleneck of In-Context Exploration.} (a) \textbf{Capacity:} Trajectory length dictates the maximum possible state coverage (Proposition \ref{propos:length_as_capacity}). (b) \textbf{``Shallow Exploration Trap'':} The probability of reaching deep states decays exponentially (Lemma \ref{lemma:shallow_exploration_trap}), preventing the model from utilizing this capacity.}
  \label{fig:motivation}
  \vspace{-1em}
\end{figure}

\paragraph{Quantifying in-context exploration.}
Furthermore, we can now quantify in-context exploration quality.
Following the count-based exploration principle from Section~\ref{sec:count-based exploration},
we define the \emph{\textbf{In-Context Distinct State Count}} $C_{\text{context}}(\boldsymbol{\tau})$ as the cardinality of unique context states visited within a trajectory:
\begin{equation} \label{eq:coverage_metric}
  C_{\text{context}}(\boldsymbol{\tau}) = \bigl| \{ \phi(s_t) \mid t = 1, \dots, L \} \bigr| \,.
\end{equation}
This metric captures the \textit{volume} of the state space covered during reasoning.

% \item \textbf{In-Context Distinct Ratio ($R_{\text{context}}$):}
%       To normalize against trajectory length, we divide the count by the total number of generated $n$-grams $M$:
%       $$ R_{\text{context}}(\boldsymbol{\tau}) = \nicefrac{C_{\text{context}}(\boldsymbol{\tau})}{M} \,. $$
%       While the Count $C_{\text{context}}$ measures volume, the Ratio $R_{\text{context}}$ measures the \textit{efficiency} of exploration.
%       A low ratio indicates low information density, suggesting the model is inflating length without generating novel information.
Section~\ref{sec:count-based exploration} posits that the optimal exploration strategy involves assigning a bonus reward proportional to the inverse of state visitation counts, i.e, $\propto 1/\sqrt{N(s)}$, which corresponds to maximizing $C_{\text{context}}$.
However, as described in Goodhart's law~\citep{goodhart1984problems} and discussed in Appendix~\ref{appendix:distinct_reward}, directly maximizing $C_{\text{context}}$ leads to losing efficacy and reward hacking in RL training.
% By maximizing $C_\text{context}$, we minimize the ``revisit'' frequency of in-context states.

% Intuitively, maximizing $C_{\text{context}}$ is the objective to enhance in-context exploration during test time.\tao{why? not fully convinced. add a motvation, e.g., from AI: Maximizing $C_{\text{context}}$ directly addresses the exploration-length paradox by encouraging the model to generate more novel local patterns within reachable trajectory lengths, effectively increasing the information density per generated token.}

% \paragraph{Global (Corpus-Level) Diversity}
% To assess the diversity across the entire set of trajectories $\mathcal{T}$, we extend these metrics globally.
% The \textbf{Global Distinct Count} is the cardinality of the union of unique $n$-grams across all samples:
% $$ C_{\text{global}}(\mathcal{T}) = \left| \bigcup_{\boldsymbol{\tau} \in \mathcal{T}} \{ g \mid g \in \mathcal{G}_{\boldsymbol{\tau}} \} \right| $$
% This quantifies the \textbf{total volume} of the \textbf{output space} covered by the model.
% Correspondingly, the \textbf{Global Distinct Ratio} is defined as the global count divided by the sum of all generated $n$-grams:
% $$ R_{\text{global}}(\mathcal{T}) = \frac{C_{\text{global}}(\mathcal{T})}{\sum_{\boldsymbol{\tau} \in \mathcal{T}} |\mathcal{G}_{\boldsymbol{\tau}}|} $$
% This evaluates whether the exploration covers diverse regions of the whole language output space.
% A high global ratio implies high inter-trajectory diversity, whereas a low ratio suggests the model collapses to similar patterns across different contexts.

\begin{figure*}[!t]
  \centering
  \centering
  \includegraphics[width=\linewidth]{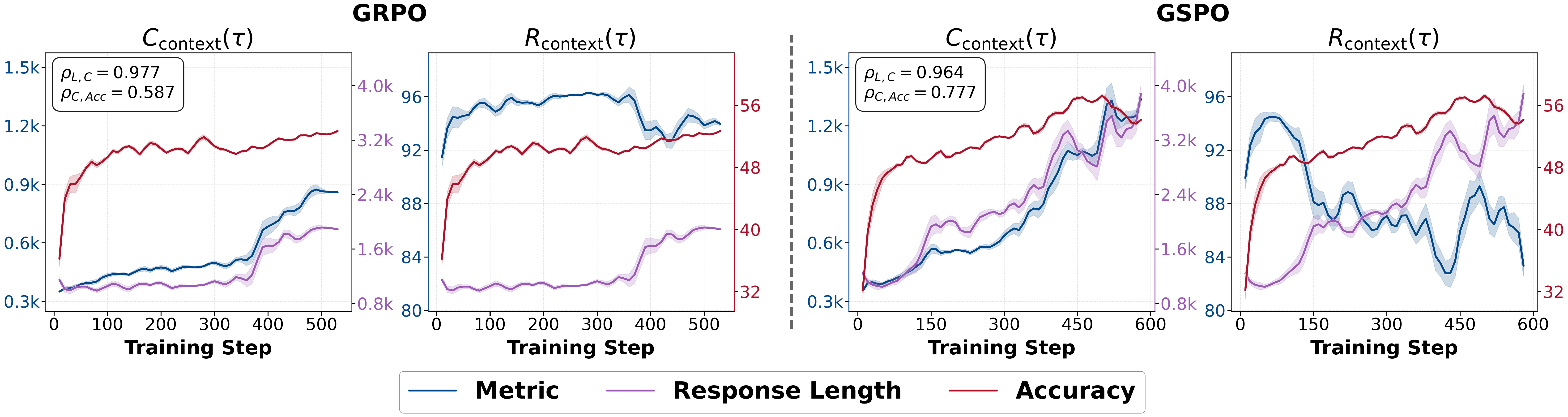}
  \caption{\textbf{The training dynamics of $C_{\text{context}}$ and $R_{\text{Context}}$ in GRPO and GSPO on Qwen3-4B-Base}.
    Two limitations are observed: (1) \textbf{Shallow Exploration Trap}: GRPO faces bottlenecks in extending trajectory length and performance, while GSPO shows slow length expansion.
    (2) \textbf{Degrading Information Density}: Both methods display degradation in ratio over time.}
  \label{fig:baseline}
  \vspace{-1em}
\end{figure*}

\subsection{The Length Bottleneck of In-Context Exploration}
% \tao{Figure 2 introduces "Exploration-Length Paradox" but this isn't clearly motivated or explained in the text. You state that "in-context exploration encourages the model to explore more states" but never explicitly state the paradox: longer trajectories allow more coverage but are harder to reach due to exponential decay.}
% Having established $C_{\text{context}}(\boldsymbol{\tau})$ as the objective for in-context exploration, we now analyze the structural constraints that affect its maximization.
Given the unreliability of direct optimization, we instead analyze a more fundamental question: \emph{what structural factor bottlenecks the growth of $C_{\text{context}}$?}

\vspace{-1em}
\paragraph{Length as the capacity for exploration.}
From the definition in Eq.~\ref{eq:state_abstraction} and Eq.~\ref{eq:coverage_metric}, we first derive a relationship between trajectory length and exploration potential:
\begin{proposition}[\textbf{Length as the Capacity for Exploration}] \label{propos:length_as_capacity}
  The cumulative exploration utility of a trajectory is strictly upper-bounded by its length $L$:
  \begin{equation}
    C_{\text{context}}(\boldsymbol{\tau}) \lesssim L \,.
  \end{equation}
  Consequently, extending the sequence length $L$ is a \textbf{necessary condition} to raise the ceiling of exploration potential.
\end{proposition}
\looseness=-1
Proposition~\ref{propos:length_as_capacity} establishes that sequence length $L$ defines the ``theoretical capacity'' for state discovery (as evidenced in Figure~\ref{fig:motivation}a).
To maximize state coverage, the model must generate sufficiently long trajectories.
\looseness=-1
\paragraph{The Shallow Exploration Trap.}
However, language models face intrinsic difficulties in sampling long sequences, as stated below.
\begin{lemma}[\textbf{Exponential Decay of Long Sequences}]  \label{lemma:shallow_exploration_trap}
  Let $\mathcal{S}_L$ denote the set of states (i.e., responses in LLM reasoning) with sentence length $L$, and $p(\mathcal{S}_L)$ denote the sampling probability for that state set.
  Then, there exists a constant $\epsilon\in (0,1)$, such that the sampling probability of $\mathcal{S}_L$ is upper-bounded by an exponentially decaying function of its length $L$, the proof can be found in Appendix ~\ref{appendix_lemma:length}: \looseness=-1
  \begin{equation}
    p(\mathcal{S}_L) < (1-\epsilon)^{L-1} \,.
  \end{equation}
\end{lemma}
Lemma~\ref{lemma:shallow_exploration_trap} characterizes the \textbf{Shallow Exploration Trap}: the model is exponentially biased against the long trajectories required to reach reasoning states (Figure~\ref{fig:motivation}b).

\begin{remark}[\textbf{The Exploration-Length Conflict}] \label{remark:conflit}
  Together, Proposition~\ref{propos:length_as_capacity} and Lemma~\ref{lemma:shallow_exploration_trap} reveal a fundamental conflict for in-context exploration: \emph{long trajectories are necessary for broad state coverage, yet they are exponentially unlikely to be sampled}.
\end{remark}

\subsection{Empirical Validation in RL Training} \label{sec:analysis_baseline}
To further validate the relationship between response length, state coverage, and model performance (Proposition ~\ref{propos:length_as_capacity}), we conduct a pilot study using GRPO and GSPO.

\vspace{-1em}
\paragraph{Length correlates with coverage and performance.}
As illustrated in Figure~\ref{fig:baseline}, we observe a remarkably high Pearson correlation ($\rho (L, C_{\text{context}}) > 0.96 $) between Response Length and $C_{\text{context}}$.
This confirms our Proposition~\ref{propos:length_as_capacity}: extending the physical length of the trajectory is the primary mechanism by which the model expands its exploration volume.
Crucially, as the model accesses these longer sequences, reasoning accuracy improves.

\vspace{-1em}
\paragraph{The limitations of standard RL Training.} \label{sec:limitation}
However, our analysis reveals bottlenecks to incentivized in-context exploration in current RL training:
\begin{enumerate}[nosep, leftmargin=12pt]
  \item \textbf{Shallow exploration trap} (Lemma~\ref{lemma:shallow_exploration_trap}): In GRPO, performance plateaus along with the response length stabilization.
        While GSPO shows continuous length growth, it is prohibitively slow.
        Both algorithms eventually succumb to the ``Shallow Exploration Trap'', failing to push $L$ far enough to reach complex states.
  \item \textbf{Degrading state density:} As shown in Figure \ref{fig:baseline}, length increases, the \textit{Distinct $C_{\text{context}}$ Ratio} ($R_{\text{context}}$) tends to decrease.
        % This indicates that standard RL incentives, while increasing length, without state density.
        This indicates that standard RL incentives increase trajectory length but fail to improve state density.
        The model fills the extra capacity with repetitive tokens rather than novel reasoning steps, degrading the efficiency of exploration.
\end{enumerate}
These limitations indicate that enhancing in-context exploration requires not only longer reasoning trajectories but also efficient and non-redundant exploration.

% \section{Methodology}
\section{Length-Incentivized Exploration} \label{sec:recipe}
In this section, we introduce a targeted reward shaping recipe --- Length-Incentivized Exploration(\method) to enforce in-context exploration in reinforcement learning.
% Based on the above theoretical and empirical insights, a natural question arises -- whether explicitly compelling the model to traverse longer reasoning sequences can break ``Shallow Exploration Trap''.
\subsection{The Length-Incentivized Reward}
We first introduce a Length-Incentivized Reward ($R_{\text{len}}$) designed to explicitly encourage the model to extend its reasoning trajectories beyond its initial policy's tendencies when it fails to reach the correct answer $R_{\text{acc} = 0}$.
% The total reward is formulated as:
% \looseness=-1
% \begin{equation}
%   R = R_{\text{acc}} + R_{\text{len}} \,.
% \end{equation}

We define a sample-wise target length $L_{\text{target}, i} = L_{\text{ref}, i} + \Delta L$, where $L_{\text{ref}, i}$ denotes the response length of the initial policy of  $i$-th sample and $\Delta L$ represents the desired least increment for exploration.
The length reward is defined as:
\begin{small}
  \begin{equation}
    \textstyle
    R_{\text{len}} =
    \begin{cases}
      0                              & L \ge L_{\text{target}}, R_{\text{acc}} = 0 \\
      - \eta (L_{\text{target}} - L) & L < L_{\text{target}}, R_{\text{acc}} = 0   \\
      0                              & R_{\text{acc}} = 1
    \end{cases} \,.
  \end{equation}
\end{small}%
This formulation creates a curriculum: the model is rewarded for extending its reasoning process beyond its usual depth whenever it cannot solve a problem immediately.

\subsection{Redundancy Reward: Enforcing Effective State Coverage}
While $R_{\text{len}}$ increases the exploration capacity $L$, it does not guarantee that this capacity is converted into effective state coverage $C_{\text{context}}$ (as revealed in Section ~\ref{sec:limitation}). Recall Proposition ~\ref{propos:length_as_capacity}: $C_{\text{context}}(\tau) \le L$.
The equality holds only when redundancy is minimized.
\looseness=-1
% To prevent the model from hacking the length reward via degenerate loops, we introduce a repetition penalty $R_{\text{rep}}$ that acts as a consistency constraint.

\vspace{-1em}
\paragraph{Redundancy Reward.}
We adopt a hard-threshold approximation of redundant in-context states that is in line with the widely used repetition penalty~\cite{yu2025rlpr}.
Let $\mathcal{N}_\tau(\phi(s_t))$ be the visitation count of state abstraction $\phi(s_t)$ in the current trajectory. The reward is defined as:
\begin{equation}
  R_{\text{red}} = -\beta \cdot \prod \mathbb{I}\left[ \mathcal{N}_\tau(\phi(s_t)) > \Theta \right],
  \label{eq:rep_penalty}
\end{equation}
where $\Theta$ is the threshold and $\beta$ is the penalty magnitude.
\looseness=-1
\paragraph{Theoretical Consistency.}
This formulation aligns with the count-based exploration bonus $b_t \propto 1/\sqrt{N_t}$ (Section \ref{sec:count-based exploration}).
% by treating states with excessive visitation ($N_t \gg 1$) as having negligible or negative utility.
Maximizing $R_{\text{len}} + R_{\text{red}}$ is equivalent to maximizing state coverage $C_{\text{context}}$, in a two-step manner.
The length reward elevates the upper bound of states, and the redundancy penalty matches the number of states given length.
\looseness=-1
% Consider the objective of maximizing the cumulative reward over a trajectory of length $L$:
% \begin{equation}
%     \max_{\tau} J(\tau) = \sum_{t=1}^{L} \left( r_{\text{len}} - \beta \cdot \mathbb{I}[s_t \in \mathcal{S}_{\text{visited}}] \right).
% \end{equation}
% Since $r_{\text{len}}$ is constant for valid tokens, maximizing $J(\tau)$ requires minimizing the penalty term $\sum \mathbb{I}[s_t \in \mathcal{S}_{\text{visited}}]$. This minimization effectively forces the policy to select $s_t$ such that $s_t \notin \tau_{<t}$, thereby maximizing the distinct state count $C_{\text{context}} \to L$. Thus, $R_{\text{rep}}$ serves as a constraint that couples trajectory lengthening directly with exploration breadth.
\vspace{-1em}
\paragraph{Final recipe.}
To prioritize solution accuracy while encouraging test-time scaling, we shape the reward to favor correct responses and longer reasoning trajectories, while penalizing redundant exploration:
\begin{equation}
  R = R_{\text{acc}} + R_{\text{len}} + \beta \cdot R_{\text{red}} \,,
  \label{eq:reward_recipe}
\end{equation}

\section{Experiments}
In this section, we empirically validate the effectiveness of \method.
We conduct comprehensive experiments to address the following core research questions:
\begin{itemize}[nosep, leftmargin=12pt]
  \item \textbf{RQ1: performance \& test-time scaling.} Does our recipe mitigate the "Shallow Exploration Trap" to achieve more effective test-time scaling?
  \item \textbf{RQ2: impact of model capability.} How does the effectiveness of our recipe on the different policy models' initial state?
  \item \textbf{RQ3: exploration dynamics.} How does the length-incentivized recipe impact global state diversity and exploration during training?
\end{itemize}

\subsection{Experimental Setup}
\paragraph{Evaluation.}
We assess performance across eight reasoning benchmarks, categorized into in-domain mathematical tasks and out-of-distribution (OOD) general reasoning.
The in-domain suite comprises AIME 2024/2025, AMC \citep{li2024numinamath}, MATH-500 \citep{hendrycks2021measuring}, and OlympiadBench \citep{he2024olympiadbench}.
The OOD evaluation includes ARC-c \citep{clark2018think}, GPQA-Diamond (denoted as GPQA* \citep{rein2024gpqa}), and MMLU-Pro \citep{wang2024mmlu}.
Regarding metrics, for benchmarks with limited sample sizes (AIME and AMC), we report the average accuracy over 32 independent runs (Avg@32).
For all other benchmarks, we report Pass@1.
All evaluations are conducted with a sampling temperature of 0.6, top-p=1.0, and a maximum response length of 32k tokens, which surpasses the training budget.

\vspace{-1em}
\paragraph{RLVR setups.}
Training is performed with a prompt batch size of 128, generating 8 rollouts per prompt.
We update the policy using a mini-batch size of 32 and a learning rate of 1e-6.
During the training phase, the sampling temperature is set to 1.0, with a maximum response length of 8,192 tokens.
Experiments are implemented using the verl framework\footnote{https://github.com/volcengine/verl} on nodes equipped with 4×H100 GPUs, employing Math-Verify\footnote{https://github.com/huggingface/Math-Verify} for outcome-based reward calculation.
We default set $n=10$, $\eta=0.3/9000$, $\beta=0.6$ and $\Theta=10$ in reward.
We discussed these hyperparameters in Appendix \ref{appendix:ablation}.

\paragraph{Models and baselines.}
We designate Qwen3-4B-Base as our primary testbed.
To verify the universality of our recipe across different training stages and model families, we also extend our evaluation to Qwen3-4B (non-thinking) and Llama-OctoThinker-3B-Base~\citep{wang2025octothinker}.
The training datasets are selected to align with the respective models: DAPO-Math-17k \citep{yu2025dapo} for Qwen3-4B-Base, Polaris \citep{Polaris2025} for Qwen3-4B-Instruct, and DeepMath-5k \citep{he2025deepmath, tan2025bottom} for Llama-OctoThinker.
We primarily investigate GSPO as the core algorithm.
Additionally, for Qwen3-4B-Base, we benchmark against standard GRPO and a GRPO variant with a high clip ratio (0.28) to validate our recipe's effectiveness across different algorithms.

\begin{table*}[h]
  \caption{\textbf{In-Domain and Out-of-Domain evaluation performance based on Qwen3-4B-Base.} \textbf{Bold} and \underline{underline} indicate the best and second-best results. Gains of our methods compared to corresponding baselines are marked in \textcolor{blue}{blue}.}
  \label{tab:main}
  \resizebox{\textwidth}{!}{%
    \begin{tabular}{lcccccccccc}
      \toprule
      \multirow{2}{*}{\textbf{Model}}       & \multicolumn{6}{c}{\textbf{In-Domain Performance}} & \multicolumn{4}{c}{\textbf{Out-of-Domain Performance}}                                                                                                                                                                                                                                     \\
                                            & \textbf{MATH}                                      & \textbf{Olympiad}                                      & \textbf{AMC}     & \textbf{AIME}    & \textbf{AIME25}  & \multicolumn{1}{c|}{\textbf{Avg.}}                               & \textbf{ARC-c}   & \textbf{GPQA}    & \textbf{MMLU-Pro} & \textbf{Avg.}                               \\ \midrule
      \textbf{Qwen3-4B-Base}                & 66.0                                               & 33.2                                                   & 36.6             & 8.5              & 6.9              & \multicolumn{1}{c|}{30.2}                                        & 66.9             & 26.3             & 30.9              & 41.4                                        \\ \midrule
      \textbf{GRPO}                         & 80.4                                               & 47.1                                                   & 55.2             & 16.8             & 18.7             & \multicolumn{1}{c|}{43.6}                                        & 84.6             & 44.4             & 60.1              & 63.0                                        \\
      \textbf{GRPO w/Clip-higher}           & 86.4                                               & \underline{54.1}                                       & 61.8             & 25.2             & 22.2             & \multicolumn{1}{c|}{49.9}                                        & 89.6             & 46.0             & 62.8              & \underline{66.1}                            \\
      \textbf{GSPO}                         & 85.2                                               & 51.7                                                   & 62.7             & 26.7             & 20.5             & \multicolumn{1}{c|}{49.4}                                        & 88.4             & \textbf{48.5}    & 61.5              & \underline{66.1}                            \\ \midrule

      \textbf{GRPO + \method}               & 85.0                                               & 49.9                                                   & 60.5             & 22.9             & 16.4             & \multicolumn{1}{c|}{$46.9_{\textcolor{blue}{+3.3}}$}             & 90.3             & 46.5             & 60.4              & $65.7_{\textcolor{blue}{+2.7}}$             \\

      \textbf{GRPO w/Clip-higher + \method} & \textbf{88.8}                                      & \underline{54.1}                                       & \underline{65.5} & \underline{30.4} & \underline{24.5} & \multicolumn{1}{c|}{$\underline{52.6}_{\textcolor{blue}{+2.7}}$} & \textbf{91.5}    & 43.9             & \underline{63.0}  &  $\underline{{66.1}}_{\textcolor{blue}{+0.0}}$                            \\

      \textbf{GSPO + \method}               & \underline{88.4}                                   & \textbf{57.2}                                          & \textbf{66.2}    & \textbf{30.5}    & \textbf{26.7}    & \multicolumn{1}{c|}{$\mathbf{53.8}_{\textcolor{blue}{+4.4}}$}    & \underline{91.4} & \underline{47.5} & \textbf{63.8}     & $\mathbf{67.6}_{\textcolor{blue}{+1.5}}$    \\ \bottomrule
    \end{tabular}%
  }
\end{table*}

\subsection{Results}

% \vspace{-1em}
\paragraph{\method achieves superior performance across algorithms on Qwen3-4B-Base.}
As shown in Table~\ref{tab:main}, our recipe consistently outperforms baselines across all benchmarks. On Qwen3-4B-Base, applying our method to GSPO boosts average in-domain accuracy from 49.4\% to 53.8\% and OOD accuracy from 62.7\% to 67.1\%. Notably, we achieve a 6.2\% gain on the challenging AIME25, confirming that breaking the ``Shallow Exploration Trap'' effectively unlocks complex problem-solving capabilities.
Moreover, our recipe achieves consistent gains across different model sizes (as shown in Appendix~\ref{appendix:different_model_size}).
% Crucially, these gains are not limited to the training distribution; the significant boosts on ARC-c (+2.3\%) and MMLU-Pro (+1.6\%) indicate that the model learns generalizable search patterns rather than overfitting to specific math problems.

\vspace{-1em}
\paragraph{Effects of recipe on in-context state coverage.}
\begin{figure}[!t]
  \centering
  \includegraphics[width=\linewidth]{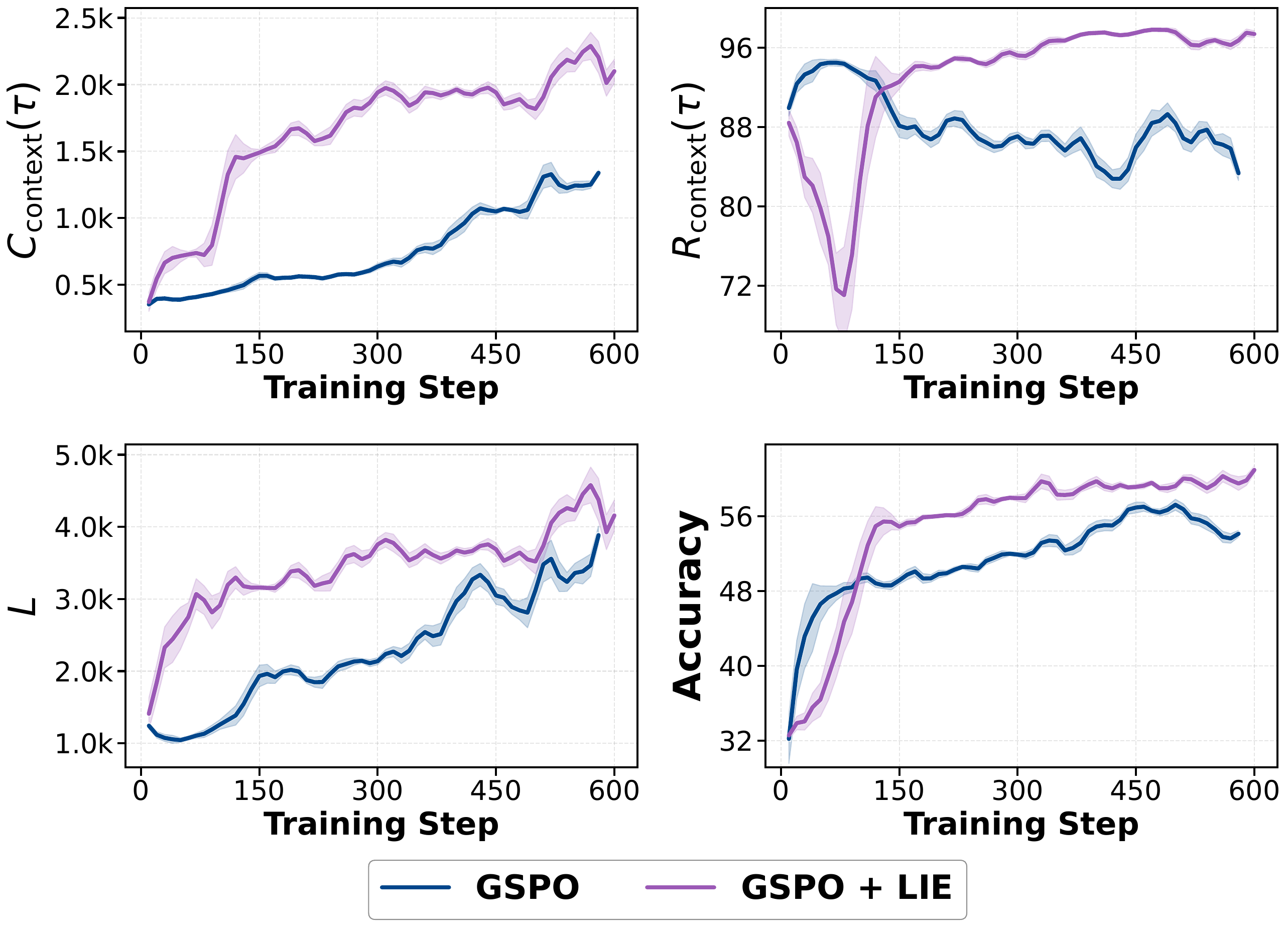}
  \caption{$C_{\text{context}}$, $R_{\text{context}}$, response length, and performance on the valid dataset comparing GSPO baseline and our recipe.}
  \label{fig:qwen3_4b_gspo_ours}
  \vspace{-1em}
\end{figure}
To understand the source of these gains, we analyze the training dynamics in Figure~\ref{fig:qwen3_4b_gspo_ours}.
Compared to the GSPO baseline, our method explicitly drives longer trajectories, resulting in a rapid expansion of $C_{\text{context}}$.
This confirms that our recipe incentive effectively increases the ``Capacity for Exploration'' (Proposition~\ref{propos:length_as_capacity}).
While the state density $R_{\text{context}}$ initially drops, the model eventually learns to utilize this extended budget, converting the increased state coverage into reasoning paths that correlate with the accuracy spike.

\vspace{-1em}
\paragraph{Test-time scaling via Long CoT.}
\begin{figure}[!t]
  \centering
  \includegraphics[width=\linewidth]{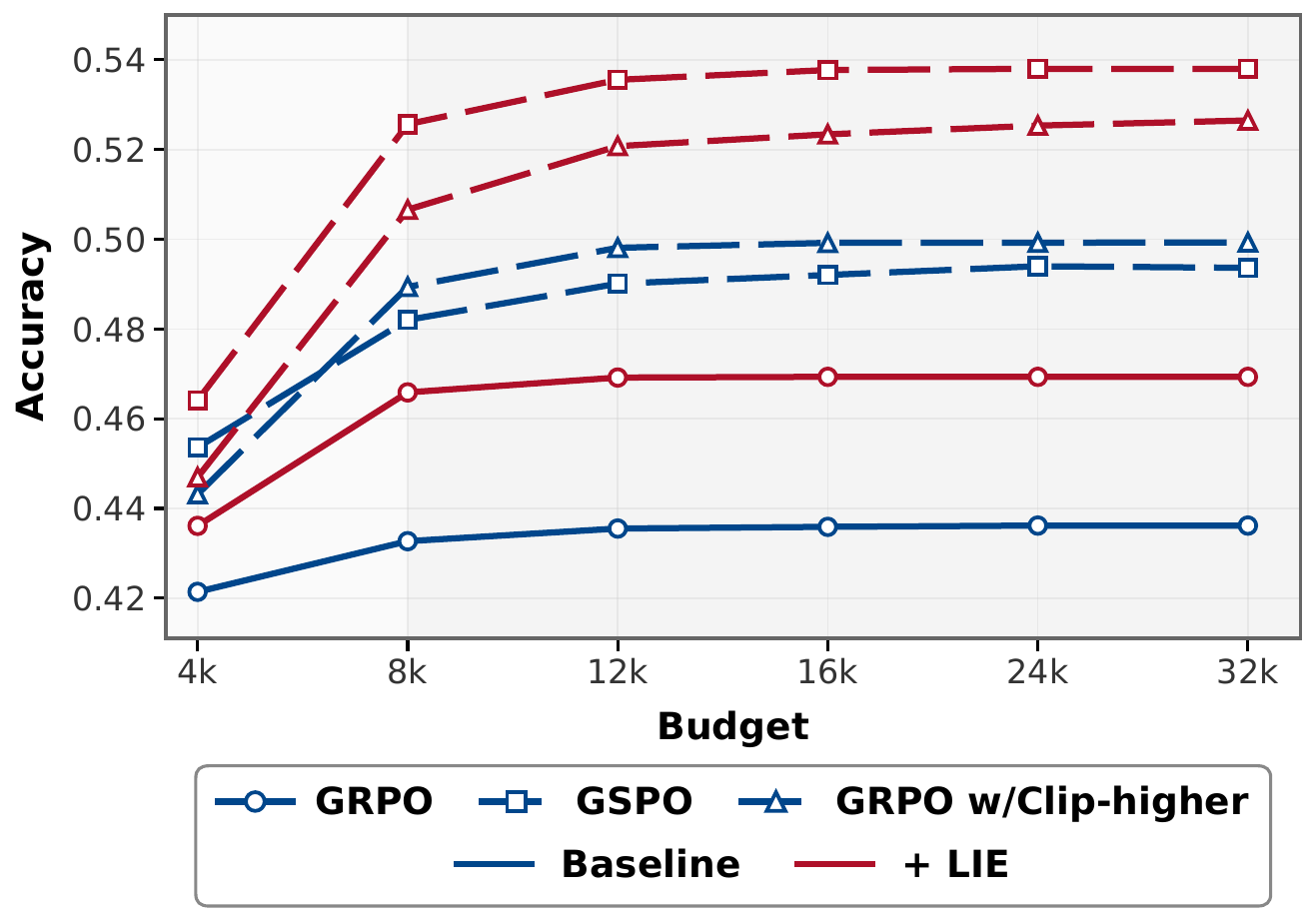}
  \caption{\textbf{Test-time extrapolation performance}. While standard baselines saturate or degrade when forced beyond their learned policy length, the Length-Incentivized Exploration recipe exhibits a superior scaling curve.}
  \label{fig:extrapolate}
  \vspace{-1em}
\end{figure}
Figure~\ref{fig:extrapolate} illustrates the performance trajectory as we increase the inference compute budget.
While standard RL models (Blue line) saturate or degrade when forced beyond their learned policy length, our recipe (Red line) exhibits a superior scaling curve, maintaining an upward trend.
This confirms that \method effectively utilizes the extended token budget to explore broader hypothesis spaces by improved in-context exploration.

\vspace{-1em}
\paragraph{\method is effective across different models.} \label{results:different_models}
\begin{table}[t]
  \caption{\textbf{In-Domain Evaluation performance based on Qwen3-4B and Llama-OctoThinker-3B.} \textbf{Bold} indicates the best result.
    Gains of our methods compared to corresponding baselines are marked in \textcolor{blue}{blue}.}
  \label{table:llama_qwen3_4b}
  \resizebox{\columnwidth}{!}{%
    \begin{tabular}{lcccccc}
      \toprule
      \multicolumn{1}{l}{}    & \textbf{MATH} & \textbf{Olympiad} & \textbf{AMC}  & \textbf{AIME} & \textbf{AIME25} & \textbf{Avg.} \\ \midrule
      \textbf{Qwen3-4B}       & 82.8          & 51.9              & 60.4          & 24.2          & 19.4            & 47.7          \\
      \textbf{GSPO}           & 94            & \textbf{68.1}     & 82            & 54.2          & 42.5            & 68.2          \\
      \textbf{GSPO + \method} & \textbf{94.4} & 67                & \textbf{85.3} & \textbf{57.7} & \textbf{46.4}   & $\mathbf{70.2}_{\textcolor{blue}{+2.0}}$ \\ \midrule
      \textbf{OctoThinker}    & 23.4          & 9.0               & 10.4          & 1.1           & 0.6             & 8.9           \\
      \textbf{GSPO}           & 55.8          & 23.1              & 28.2          & 3.8           & 2.3             & 22.6          \\
      \textbf{GSPO + \method} & \textbf{60.8} & \textbf{28.1}     & \textbf{30.3} & \textbf{4.5}  & \textbf{4.4}    & $\mathbf{25.6}_{\textcolor{blue}{+3.0}}$ \\ \bottomrule
    \end{tabular}%
  }
  % \vspace{-1em}
\end{table}

\begin{figure*}[!t]
  \centering
  \includegraphics[width=\linewidth]{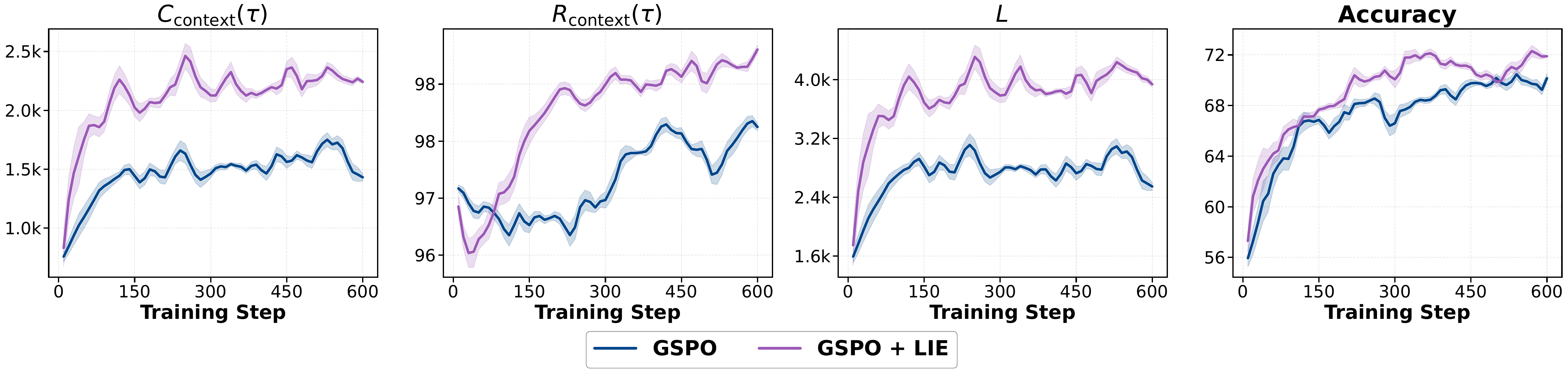}
  \caption{
    Tracking $C_{\text{context}}$, $R_{\text{context}}$, response length $L$, and valid performance during the training of Qwen3-4B.
    Compared to the GSPO baseline (\textcolor[HTML]{00468B}{Blue}),
    the LIE (\textcolor[HTML]{9B59B6}{Purple})
    successfully breaks the ``Shallow Exploration Trap'' and achieves performance gains.
  }
  \label{fig:qwen3_4b}
  \vspace{-1em}
\end{figure*}
We extended our recipe to Qwen3-4B (post-trained) and Llama-OctoThinker~\citep{wang2025octothinker}.
Table~\ref{table:llama_qwen3_4b} reports consistent gains (2\%--3\%) across models.
The training dynamics in Figures~\ref{fig:llama} and~\ref{fig:qwen3_4b} reveal that our recipe universally drives the expansion of state coverage, converting the increased trajectory length into improved in-context exploration.

\subsection{Analysis and Discussion}

\paragraph{Global diversity.}
\begin{figure}[h]
  \centering
  \includegraphics[width=\linewidth]{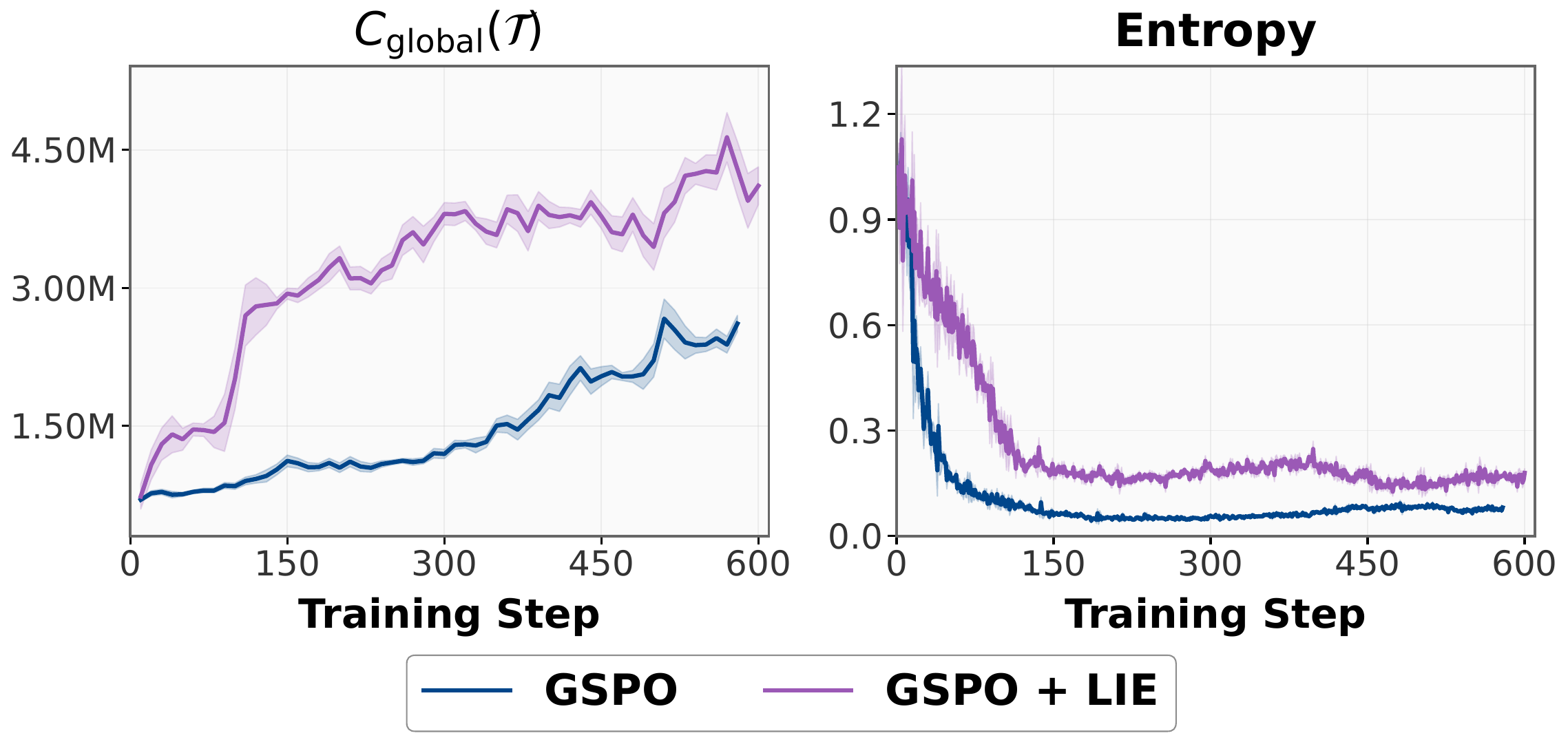}
  \caption{\textbf{Global Exploration Dynamics: Diversity and Entropy}. Our recipe maintains significantly higher entropy and continuous growth in global state coverage compared with the baseline.}
  \label{fig:global_diversity}
  \vspace{-1em}
\end{figure}

Beyond in-context state diversity, we track Global Distinct Count ($C_\text{global}$) and policy entropy to assess exploration across the state space during training time.
As shown in Figure~\ref{fig:global_diversity}, standard GSPO exhibits saturation and a rapid entropy drop, indicating premature convergence (mode collapse).
In contrast, our recipe maintains continuous growth in and significantly higher entropy.
This sustained global exploration prevents mode collapse, enabling the policy to discover rare, high-reward states during training that standard methods miss.

\vspace{-1em}
\paragraph{Continual scaling via curriculum training.}
\begin{table}[t]
  \caption{\textbf{Continual Scaling via Curriculum Training based on Qwen3-4B-Base.} ``Stage 1'' represents the initial training, while ``Stage 2'' employs further length extension. The results demonstrate consistent improvements across benchmarks.}
  \label{tab:continual}
  \resizebox{\columnwidth}{!}{%
    \begin{tabular}{lcccccc}
      \toprule
      \multicolumn{1}{l}{}   & \textbf{MATH} & \textbf{Olympiad} & \textbf{AMC} & \textbf{AIME} & \textbf{AIME25} & \textbf{Avg.}                   \\ \midrule
      \textbf{Qwen3-4B-Base} & 66.0          & 33.2              & 36.6         & 8.5           & 6.9             & 30.2                            \\
      \textbf{GSPO}          & 85.2          & 51.7              & 62.7         & 26.7          & 20.5            & 49.4                            \\
      \textbf{Stage1}        & 88.4          & 57.2              & 66.2         & 30.5          & 26.7            & $53.8_{\textcolor{blue}{+4.4}}$ \\
      \textbf{Stage2}        & 89.4          & 59.0                & 71.7         & 33.4          & 27.9            & $56.3_{\textcolor{blue}{+2.5}}$ \\ \bottomrule
    \end{tabular}%
  }
  \vspace{-1em}
\end{table}
To validate scaling laws beyond the initial setup, we conducted a second training stage (``Stage 2'') with a relaxed 12k token limit.
As shown in Table~\ref{tab:continual}, the model exhibits continuous improvement.
This confirms that our recipe effectively converts additional compute into reasoning accuracy, establishing a scaling trend with the allocated reasoning horizon.

\vspace{-1em}
\paragraph{Changes in Reasoning Behaviors.}
\begin{figure}
  \centering
  \includegraphics[width=0.86\linewidth]{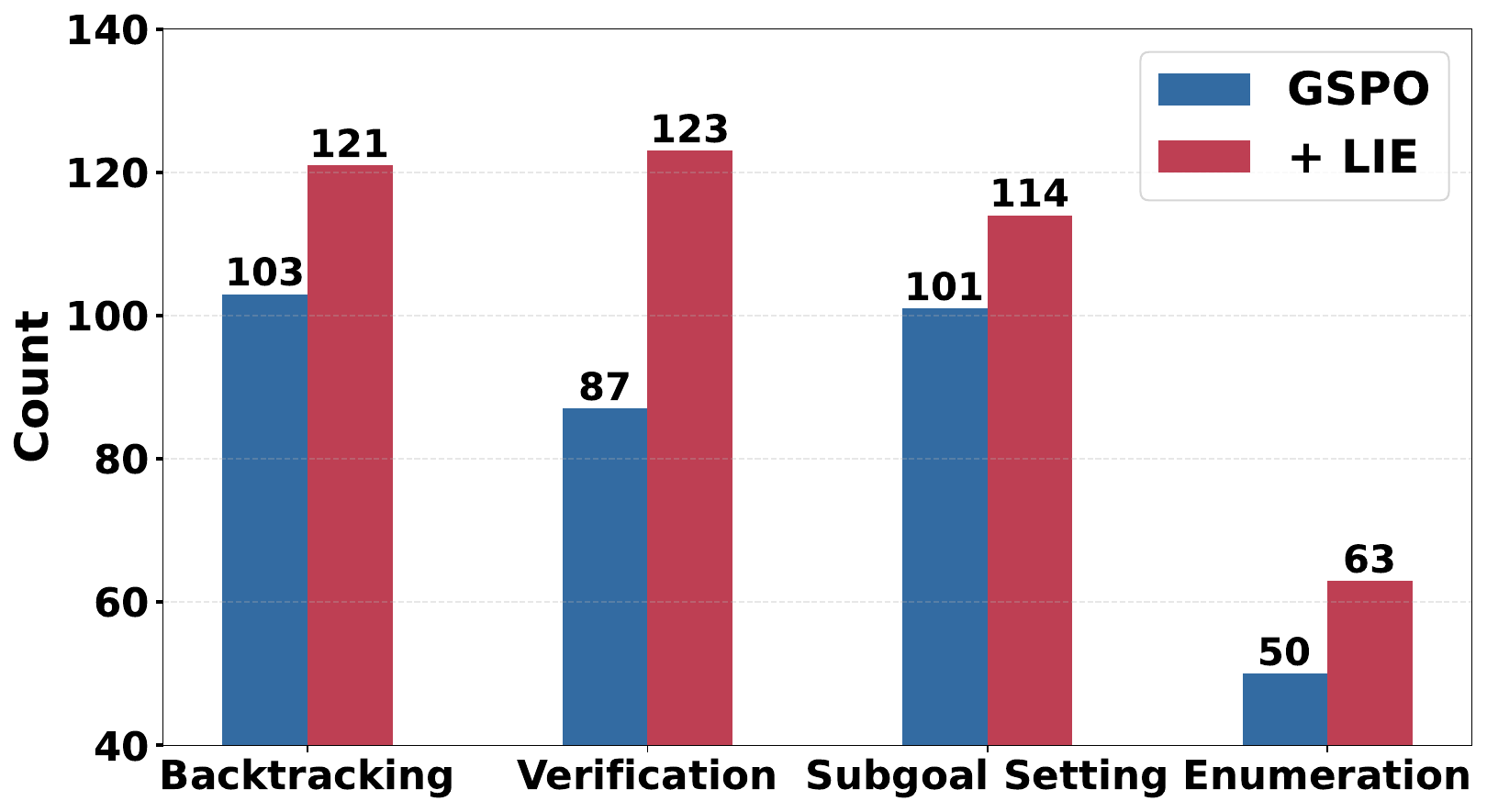}
  \caption{\textbf{Frequency analysis of cognitive behaviors}.
    The proposed recipe (+ LIE) significantly boosts all behaviors, with a notable enhancement in Backtracking.}
  \label{fig:cognitive_behaviors}
  \vspace{-1em}
\end{figure}
Adopting the framework from \citep{gandhi2025cognitive}, we employ Claude-3.7-Thinking to quantify four representative cognitive behaviors: Backtracking, Verification, Subgoal Setting, and Enumeration.
As illustrated in Figure ~\ref{fig:cognitive_behaviors}, the integration of LIE consistently stimulates more frequent reasoning activities across all categories.
Crucially, we observe a substantial increase in Backtracking (from 103 to 121), which is a sophisticated cognitive capability.
This improvement underscores that LIE effectively enhances the model's capacity for in-context exploration.
We also discuss the changes in reasoning structure in Appendix~\ref{appendix:case_study}.
\vspace{-1em}
\paragraph{Only $R_{\text{len}}$ incentives exploration but introduce repetition.}
As shown in Figure~\ref{fig:preliminary}, explicitly incentivizing longer trajectories rapidly expands $C_{\text{context}}$ and improves accuracy, confirming that length is capacity for additional in-context exploration (Proposition~\ref{propos:length_as_capacity}).
However, this expansion is accompanied by a decline in $R_{\text{context}}$ revealing that the model often satisfies the length reward through redundant rather than informative generation.
This observation necessitates a redundancy penalty to maintain exploration quality, a synergy further detailed in the ablation study of our reward components (Equation~\ref{eq:reward_recipe}) in Appendix~\ref{appendix:reward_ablation}.

% Together, these results suggest that length incentives alone are sufficient to overcome shallow exploration and unlock additional performance gains, but they do so inefficiently.
% Without further regularization, increased computation tends to translate into redundant exploration.

\begin{figure}[!t]
  \centering
  \includegraphics[width=\linewidth]{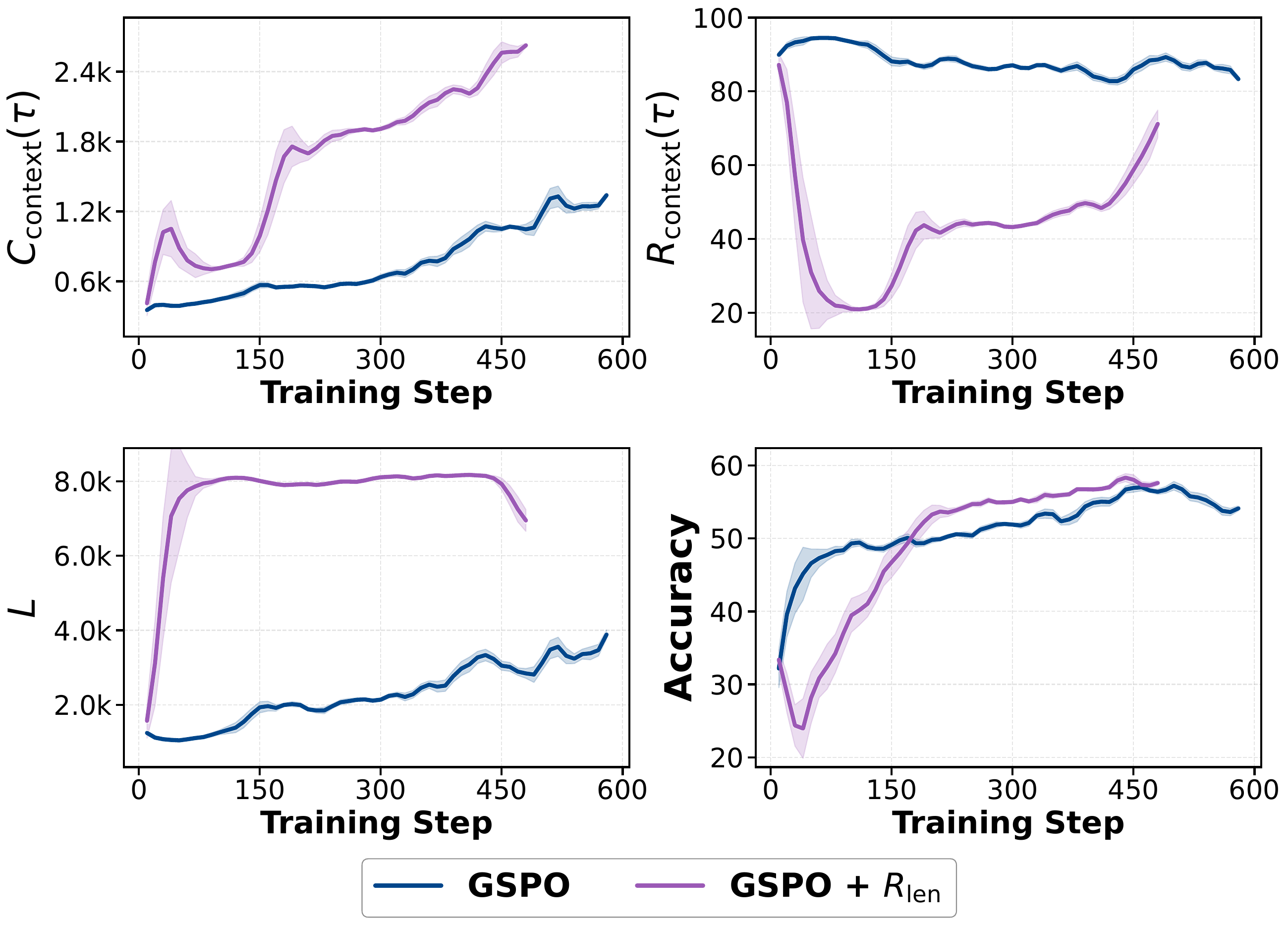}
  \caption{Impact of the Length-Incentivized Reward ($L_{\text{len}}$). Forcing the model to "think longer" yields performance gains. However, the accompanying drop in the distinct ratio $R_{\text{context}}$ reveals a tendency towards repetition.}
  \label{fig:preliminary}
  \vspace{-2em}
\end{figure}

\section{Conclusion}
In this work, we identified the conflict between the length required for broader state coverage and the ``Shallow Exploration Trap'' as a critical bottleneck in in-context exploration through theoretical and empirical analysis.
To overcome this limitation, we proposed Length-Incentivized Exploration, a simple yet effective recipe that maximizing the in-context state coverage.
Our experiments across diverse models and benchmarks confirm that \method incentivize the in-context exploration during test time.
Overall, our work advances test-time scaling by explicitly activating in-context exploration, demonstrating how additional computation can be systematically converted into more effective reasoning.

\section*{Impact Statements}
This paper presents work whose goal is to advance the field of machine learning.
There are many potential societal consequences of our work, none of which we feel must be specifically highlighted here.

% In the unusual situation where you want a paper to appear in the
% references without citing it in the main text, use \nocite
\nocite{langley00}

\bibliography{sources/reference}

@article{guo2025deepseek,
  title   = {Deepseek-r1: Incentivizing reasoning capability in llms via reinforcement learning},
  author  = {Guo, Daya and Yang, Dejian and Zhang, Haowei and Song, Junxiao and Zhang, Ruoyu and Xu, Runxin and Zhu, Qihao and Ma, Shirong and Wang, Peiyi and Bi, Xiao and others},
  journal = {arXiv preprint arXiv:2501.12948},
  year    = {2025}
}

@article{setlur2025e3,
  title   = {e3: Learning to Explore Enables Extrapolation of Test-Time Compute for LLMs},
  author  = {Setlur, Amrith and Yang, Matthew YR and Snell, Charlie and Greer, Jeremy and Wu, Ian and Smith, Virginia and Simchowitz, Max and Kumar, Aviral},
  journal = {arXiv preprint arXiv:2506.09026},
  year    = {2025}
}

@inproceedings{muennighoff2025s1,
  title     = {s1: Simple test-time scaling},
  author    = {Muennighoff, Niklas and Yang, Zitong and Shi, Weijia and Li, Xiang Lisa and Fei-Fei, Li and Hajishirzi, Hannaneh and Zettlemoyer, Luke and Liang, Percy and Cand{\`e}s, Emmanuel and Hashimoto, Tatsunori B},
  booktitle = {Proceedings of the 2025 Conference on Empirical Methods in Natural Language Processing},
  pages     = {20286--20332},
  year      = {2025}
}

@article{yang2025qwen3,
  title   = {Qwen3 technical report},
  author  = {Yang, An and Li, Anfeng and Yang, Baosong and Zhang, Beichen and Hui, Binyuan and Zheng, Bo and Yu, Bowen and Gao, Chang and Huang, Chengen and Lv, Chenxu and others},
  journal = {arXiv preprint arXiv:2505.09388},
  year    = {2025}
}

@article{kumar2024training,
  title   = {Training language models to self-correct via reinforcement learning},
  author  = {Kumar, Aviral and Zhuang, Vincent and Agarwal, Rishabh and Su, Yi and Co-Reyes, John D and Singh, Avi and Baumli, Kate and Iqbal, Shariq and Bishop, Colton and Roelofs, Rebecca and others},
  journal = {arXiv preprint arXiv:2409.12917},
  year    = {2024}
}

@misc{deepscaler2025,
  title        = {DeepScaleR: Surpassing 01-Preview with a 1.5B Model by Scaling RL},
  author       = {Michael Luo and Sijun Tan and Justin Wong and Xiaoxiang Shi and William Y. Tang and Manan Roongta and Colin Cai and Jeffrey Luo and Li Erran Li and Raluca Ada Popa and Ion Stoica},
  howpublished = {\url{https://pretty-radio-b75.notion.site/DeepScaleR-Surpassing-01-Preview-with-a-1-5B-Model-by-Scaling-RL-19681902c1468005bed8}},
  year         = {2025}
}

@article{gandhi2025cognitive,
  title   = {Cognitive behaviors that enable self-improving reasoners, or, four habits of highly effective stars},
  author  = {Gandhi, Kanishk and Chakravarthy, Ayush and Singh, Anikait and Lile, Nathan and Goodman, Noah D},
  journal = {arXiv preprint arXiv:2503.01307},
  year    = {2025}
}

@article{team2025kimi,
  title   = {Kimi k1. 5: Scaling reinforcement learning with llms},
  author  = {Team, Kimi and Du, Angang and Gao, Bofei and Xing, Bowei and Jiang, Changjiu and Chen, Cheng and Li, Cheng and Xiao, Chenjun and Du, Chenzhuang and Liao, Chonghua and others},
  journal = {arXiv preprint arXiv:2501.12599},
  year    = {2025}
}

@article{jaech2024openai,
  title   = {Openai o1 system card},
  author  = {Jaech, Aaron and Kalai, Adam and Lerer, Adam and Richardson, Adam and El-Kishky, Ahmed and Low, Aiden and Helyar, Alec and Madry, Aleksander and Beutel, Alex and Carney, Alex and others},
  journal = {arXiv preprint arXiv:2412.16720},
  year    = {2024}
}

@article{auer2002finite,
  title     = {Finite-time analysis of the multiarmed bandit problem},
  author    = {Auer, Peter and Cesa-Bianchi, Nicolo and Fischer, Paul},
  journal   = {Machine learning},
  volume    = {47},
  number    = {2},
  pages     = {235--256},
  year      = {2002},
  publisher = {Springer}
}

@article{snell2024scaling,
  title   = {Scaling llm test-time compute optimally can be more effective than scaling model parameters},
  author  = {Snell, Charlie and Lee, Jaehoon and Xu, Kelvin and Kumar, Aviral},
  journal = {arXiv preprint arXiv:2408.03314},
  year    = {2024}
}

@article{cui2025entropy,
  title   = {The entropy mechanism of reinforcement learning for reasoning language models},
  author  = {Cui, Ganqu and Zhang, Yuchen and Chen, Jiacheng and Yuan, Lifan and Wang, Zhi and Zuo, Yuxin and Li, Haozhan and Fan, Yuchen and Chen, Huayu and Chen, Weize and others},
  journal = {arXiv preprint arXiv:2505.22617},
  year    = {2025}
}

@article{wang2025beyond,
  title   = {Beyond the 80/20 rule: High-entropy minority tokens drive effective reinforcement learning for llm reasoning},
  author  = {Wang, Shenzhi and Yu, Le and Gao, Chang and Zheng, Chujie and Liu, Shixuan and Lu, Rui and Dang, Kai and Chen, Xionghui and Yang, Jianxin and Zhang, Zhenru and others},
  journal = {arXiv preprint arXiv:2506.01939},
  year    = {2025}
}

@article{zhu2025surprising,
  title   = {The surprising effectiveness of negative reinforcement in LLM reasoning},
  author  = {Zhu, Xinyu and Xia, Mengzhou and Wei, Zhepei and Chen, Wei-Lin and Chen, Danqi and Meng, Yu},
  journal = {arXiv preprint arXiv:2506.01347},
  year    = {2025}
}

@article{jiang2025think,
  title   = {Think only when you need with large hybrid-reasoning models},
  author  = {Jiang, Lingjie and Wu, Xun and Huang, Shaohan and Dong, Qingxiu and Chi, Zewen and Dong, Li and Zhang, Xingxing and Lv, Tengchao and Cui, Lei and Wei, Furu},
  journal = {arXiv preprint arXiv:2505.14631},
  year    = {2025}
}

@article{zhang2025alphaone,
  title   = {AlphaOne: Reasoning Models Thinking Slow and Fast at Test Time},
  author  = {Zhang, Junyu and Dong, Runpei and Wang, Han and Ning, Xuying and Geng, Haoran and Li, Peihao and He, Xialin and Bai, Yutong and Malik, Jitendra and Gupta, Saurabh and others},
  journal = {arXiv preprint arXiv:2505.24863},
  year    = {2025}
}

@article{wen2025budgetthinker,
  title   = {Budgetthinker: Empowering budget-aware llm reasoning with control tokens},
  author  = {Wen, Hao and Wu, Xinrui and Sun, Yi and Zhang, Feifei and Chen, Liye and Wang, Jie and Liu, Yunxin and Liu, Yunhao and Zhang, Ya-Qin and Li, Yuanchun},
  journal = {arXiv preprint arXiv:2508.17196},
  year    = {2025}
}

@article{aggarwal2025l1,
  title   = {L1: Controlling how long a reasoning model thinks with reinforcement learning},
  author  = {Aggarwal, Pranjal and Welleck, Sean},
  journal = {arXiv preprint arXiv:2503.04697},
  year    = {2025}
}

@article{huang2025adactrl,
  title   = {AdaCtrl: Towards Adaptive and Controllable Reasoning via Difficulty-Aware Budgeting},
  author  = {Huang, Shijue and Wang, Hongru and Zhong, Wanjun and Su, Zhaochen and Feng, Jiazhan and Cao, Bowen and Fung, Yi R},
  journal = {arXiv preprint arXiv:2505.18822},
  year    = {2025}
}

@article{zhang2025adaptthink,
  title   = {Adaptthink: Reasoning models can learn when to think},
  author  = {Zhang, Jiajie and Lin, Nianyi and Hou, Lei and Feng, Ling and Li, Juanzi},
  journal = {arXiv preprint arXiv:2505.13417},
  year    = {2025}
}

@inproceedings{kang2025c3ot,
  title     = {C3ot: Generating shorter chain-of-thought without compromising effectiveness},
  author    = {Kang, Yu and Sun, Xianghui and Chen, Liangyu and Zou, Wei},
  booktitle = {Proceedings of the AAAI Conference on Artificial Intelligence},
  volume    = {39},
  number    = {23},
  pages     = {24312--24320},
  year      = {2025}
}

@article{zhang2025othink,
  title   = {OThink-R1: Intrinsic Fast/Slow Thinking Mode Switching for Over-Reasoning Mitigation},
  author  = {Zhang, Shengjia and Wu, Junjie and Chen, Jiawei and Zhang, Changwang and Lou, Xingyu and Zhou, Wangchunshu and Zhou, Sheng and Wang, Can and Wang, Jun},
  journal = {arXiv preprint arXiv:2506.02397},
  year    = {2025}
}

@article{yu2025think,
  title   = {Think smarter not harder: Adaptive reasoning with inference aware optimization},
  author  = {Yu, Zishun and Xu, Tengyu and Jin, Di and Sankararaman, Karthik Abinav and He, Yun and Zhou, Wenxuan and Zeng, Zhouhao and Helenowski, Eryk and Zhu, Chen and Wang, Sinong and others},
  journal = {arXiv preprint arXiv:2501.17974},
  year    = {2025}
}

@article{xu2025scalable,
  title   = {Scalable chain of thoughts via elastic reasoning},
  author  = {Xu, Yuhui and Dong, Hanze and Wang, Lei and Sahoo, Doyen and Li, Junnan and Xiong, Caiming},
  journal = {arXiv preprint arXiv:2505.05315},
  year    = {2025}
}

@article{ma2025cot,
  title   = {Cot-valve: Length-compressible chain-of-thought tuning},
  author  = {Ma, Xinyin and Wan, Guangnian and Yu, Runpeng and Fang, Gongfan and Wang, Xinchao},
  journal = {arXiv preprint arXiv:2502.09601},
  year    = {2025}
}

@article{yeo2025demystifying,
  title   = {Demystifying long chain-of-thought reasoning in llms},
  author  = {Yeo, Edward and Tong, Yuxuan and Niu, Morry and Neubig, Graham and Yue, Xiang},
  journal = {arXiv preprint arXiv:2502.03373},
  year    = {2025}
}

@article{cheng2025reasoning,
  title   = {Reasoning with exploration: An entropy perspective},
  author  = {Cheng, Daixuan and Huang, Shaohan and Zhu, Xuekai and Dai, Bo and Zhao, Wayne Xin and Zhang, Zhenliang and Wei, Furu},
  journal = {arXiv preprint arXiv:2506.14758},
  year    = {2025}
}

@article{yang2025towards,
  title   = {Towards thinking-optimal scaling of test-time compute for llm reasoning},
  author  = {Yang, Wenkai and Ma, Shuming and Lin, Yankai and Wei, Furu},
  journal = {arXiv preprint arXiv:2502.18080},
  year    = {2025}
}

@article{yu2025dapo,
  title   = {Dapo: An open-source llm reinforcement learning system at scale},
  author  = {Yu, Qiying and Zhang, Zheng and Zhu, Ruofei and Yuan, Yufeng and Zuo, Xiaochen and Yue, Yu and Dai, Weinan and Fan, Tiantian and Liu, Gaohong and Liu, Lingjun and others},
  journal = {arXiv preprint arXiv:2503.14476},
  year    = {2025}
}

@article{hendrycks2021measuring,
  title   = {Measuring mathematical problem solving with the math dataset},
  author  = {Hendrycks, Dan and Burns, Collin and Kadavath, Saurav and Arora, Akul and Basart, Steven and Tang, Eric and Song, Dawn and Steinhardt, Jacob},
  journal = {arXiv preprint arXiv:2103.03874},
  year    = {2021}
}

@article{li2024numinamath,
  title   = {Numinamath: The largest public dataset in ai4maths with 860k pairs of competition math problems and solutions},
  author  = {Li, Jia and Beeching, Edward and Tunstall, Lewis and Lipkin, Ben and Soletskyi, Roman and Huang, Shengyi and Rasul, Kashif and Yu, Longhui and Jiang, Albert Q and Shen, Ziju and others},
  journal = {Hugging Face repository},
  volume  = {13},
  number  = {9},
  pages   = {9},
  year    = {2024}
}

@inproceedings{he2024olympiadbench,
  title     = {Olympiadbench: A challenging benchmark for promoting agi with olympiad-level bilingual multimodal scientific problems},
  author    = {He, Chaoqun and Luo, Renjie and Bai, Yuzhuo and Hu, Shengding and Thai, Zhen and Shen, Junhao and Hu, Jinyi and Han, Xu and Huang, Yujie and Zhang, Yuxiang and others},
  booktitle = {Proceedings of the 62nd Annual Meeting of the Association for Computational Linguistics (Volume 1: Long Papers)},
  pages     = {3828--3850},
  year      = {2024}
}

@article{gandhi2024stream,
  title   = {Stream of search (sos): Learning to search in language},
  author  = {Gandhi, Kanishk and Lee, Denise and Grand, Gabriel and Liu, Muxin and Cheng, Winson and Sharma, Archit and Goodman, Noah D},
  journal = {arXiv preprint arXiv:2404.03683},
  year    = {2024}
}

@article{wang2025octothinker,
  title   = {Octothinker: Mid-training incentivizes reinforcement learning scaling},
  author  = {Wang, Zengzhi and Zhou, Fan and Li, Xuefeng and Liu, Pengfei},
  journal = {arXiv preprint arXiv:2506.20512},
  year    = {2025}
}

@article{auer2002using,
  title   = {Using confidence bounds for exploitation-exploration trade-offs},
  author  = {Auer, Peter},
  journal = {Journal of machine learning research},
  volume  = {3},
  number  = {Nov},
  pages   = {397--422},
  year    = {2002}
}

@article{wu2024inference,
  title   = {Inference scaling laws: An empirical analysis of compute-optimal inference for problem-solving with language models},
  author  = {Wu, Yangzhen and Sun, Zhiqing and Li, Shanda and Welleck, Sean and Yang, Yiming},
  journal = {arXiv preprint arXiv:2408.00724},
  year    = {2024}
}

@article{liu2025can,
  title   = {Can 1b llm surpass 405b llm? rethinking compute-optimal test-time scaling},
  author  = {Liu, Runze and Gao, Junqi and Zhao, Jian and Zhang, Kaiyan and Li, Xiu and Qi, Biqing and Ouyang, Wanli and Zhou, Bowen},
  journal = {arXiv preprint arXiv:2502.06703},
  year    = {2025}
}

@article{brown2024large,
  title   = {Large language monkeys: Scaling inference compute with repeated sampling},
  author  = {Brown, Bradley and Juravsky, Jordan and Ehrlich, Ryan and Clark, Ronald and Le, Quoc V and R{\'e}, Christopher and Mirhoseini, Azalia},
  journal = {arXiv preprint arXiv:2407.21787},
  year    = {2024}
}

@article{wang2022self,
  title   = {Self-consistency improves chain of thought reasoning in language models},
  author  = {Wang, Xuezhi and Wei, Jason and Schuurmans, Dale and Le, Quoc and Chi, Ed and Narang, Sharan and Chowdhery, Aakanksha and Zhou, Denny},
  journal = {arXiv preprint arXiv:2203.11171},
  year    = {2022}
}

@article{agarwal2025gpt,
  title   = {gpt-oss-120b \& gpt-oss-20b model card},
  author  = {Agarwal, Sandhini and Ahmad, Lama and Ai, Jason and Altman, Sam and Applebaum, Andy and Arbus, Edwin and Arora, Rahul K and Bai, Yu and Baker, Bowen and Bao, Haiming and others},
  journal = {arXiv preprint arXiv:2508.10925},
  year    = {2025}
}

@inproceedings{lightman2023let,
  title     = {Let's verify step by step},
  author    = {Lightman, Hunter and Kosaraju, Vineet and Burda, Yuri and Edwards, Harrison and Baker, Bowen and Lee, Teddy and Leike, Jan and Schulman, John and Sutskever, Ilya and Cobbe, Karl},
  booktitle = {The Twelfth International Conference on Learning Representations},
  year      = {2023}
}

@article{yu2025rlpr,
  title   = {RLPR: Extrapolating RLVR to General Domains without Verifiers},
  author  = {Yu, Tianyu and Ji, Bo and Wang, Shouli and Yao, Shu and Wang, Zefan and Cui, Ganqu and Yuan, Lifan and Ding, Ning and Yao, Yuan and Liu, Zhiyuan and others},
  journal = {arXiv preprint arXiv:2506.18254},
  year    = {2025}
}

@article{zheng2025group,
  title   = {Group sequence policy optimization},
  author  = {Zheng, Chujie and Liu, Shixuan and Li, Mingze and Chen, Xiong-Hui and Yu, Bowen and Gao, Chang and Dang, Kai and Liu, Yuqiong and Men, Rui and Yang, An and others},
  journal = {arXiv preprint arXiv:2507.18071},
  year    = {2025}
}

@article{sutton1999policy,
  title   = {Policy gradient methods for reinforcement learning with function approximation},
  author  = {Sutton, Richard S and McAllester, David and Singh, Satinder and Mansour, Yishay},
  journal = {Advances in neural information processing systems},
  volume  = {12},
  year    = {1999}
}

@article{shao2024deepseekmath,
  title   = {Deepseekmath: Pushing the limits of mathematical reasoning in open language models},
  author  = {Shao, Zhihong and Wang, Peiyi and Zhu, Qihao and Xu, Runxin and Song, Junxiao and Bi, Xiao and Zhang, Haowei and Zhang, Mingchuan and Li, YK and Wu, Yang and others},
  journal = {arXiv preprint arXiv:2402.03300},
  year    = {2024}
}

@article{clark2018think,
  title   = {Think you have solved question answering? try arc, the ai2 reasoning challenge},
  author  = {Clark, Peter and Cowhey, Isaac and Etzioni, Oren and Khot, Tushar and Sabharwal, Ashish and Schoenick, Carissa and Tafjord, Oyvind},
  journal = {arXiv preprint arXiv:1803.05457},
  year    = {2018}
}

@inproceedings{rein2024gpqa,
  title     = {Gpqa: A graduate-level google-proof q\&a benchmark},
  author    = {Rein, David and Hou, Betty Li and Stickland, Asa Cooper and Petty, Jackson and Pang, Richard Yuanzhe and Dirani, Julien and Michael, Julian and Bowman, Samuel R},
  booktitle = {First Conference on Language Modeling},
  year      = {2024}
}

@article{wang2024mmlu,
  title   = {Mmlu-pro: A more robust and challenging multi-task language understanding benchmark},
  author  = {Wang, Yubo and Ma, Xueguang and Zhang, Ge and Ni, Yuansheng and Chandra, Abhranil and Guo, Shiguang and Ren, Weiming and Arulraj, Aaran and He, Xuan and Jiang, Ziyan and others},
  journal = {Advances in Neural Information Processing Systems},
  volume  = {37},
  pages   = {95266--95290},
  year    = {2024}
}

@misc{Polaris2025,
  title  = {POLARIS: A Post-Training Recipe for Scaling Reinforcement Learning on Advanced Reasoning Models},
  url    = {https://hkunlp.github.io/blog/2025/Polaris},
  author = {An, Chenxin and Xie, Zhihui and Li, Xiaonan and Li, Lei and Zhang, Jun and Gong, Shansan and Zhong, Ming and Xu, Jingjing and Qiu, Xipeng and Wang, Mingxuan and Kong, Lingpeng},
  year   = {2025}
}

@article{he2025deepmath,
  title   = {Deepmath-103k: A large-scale, challenging, decontaminated, and verifiable mathematical dataset for advancing reasoning},
  author  = {He, Zhiwei and Liang, Tian and Xu, Jiahao and Liu, Qiuzhi and Chen, Xingyu and Wang, Yue and Song, Linfeng and Yu, Dian and Liang, Zhenwen and Wang, Wenxuan and others},
  journal = {arXiv preprint arXiv:2504.11456},
  year    = {2025}
}

@article{tan2025bottom,
  title   = {Bottom-up Policy Optimization: Your Language Model Policy Secretly Contains Internal Policies},
  author  = {Tan, Yuqiao and Wang, Minzheng and He, Shizhu and Liao, Huanxuan and Zhao, Chengfeng and Lu, Qiunan and Liang, Tian and Zhao, Jun and Liu, Kang},
  journal = {arXiv preprint arXiv:2512.19673},
  year    = {2025}
}

@article{levy2014neural,
  title   = {Neural word embedding as implicit matrix factorization},
  author  = {Levy, Omer and Goldberg, Yoav},
  journal = {Advances in neural information processing systems},
  volume  = {27},
  year    = {2014}
}

@misc{liu2025learnreasonefficientlyadaptive,
  title         = {Learn to Reason Efficiently with Adaptive Length-based Reward Shaping},
  author        = {Wei Liu and Ruochen Zhou and Yiyun Deng and Yuzhen Huang and Junteng Liu and Yuntian Deng and Yizhe Zhang and Junxian He},
  year          = {2025},
  eprint        = {2505.15612},
  archiveprefix = {arXiv},
  primaryclass  = {cs.CL},
  url           = {https://arxiv.org/abs/2505.15612}
}

@article{schulman2017proximal,
  title   = {Proximal policy optimization algorithms},
  author  = {Schulman, John and Wolski, Filip and Dhariwal, Prafulla and Radford, Alec and Klimov, Oleg},
  journal = {arXiv preprint arXiv:1707.06347},
  year    = {2017}
}

@article{feng2025characterizes,
  title   = {What characterizes effective reasoning? revisiting length, review, and structure of cot},
  author  = {Feng, Yunzhen and Kempe, Julia and Zhang, Cheng and Jain, Parag and Hartshorn, Anthony},
  journal = {arXiv preprint arXiv:2509.19284},
  year    = {2025}
}

@article{meta2024llama,
  title   = {Llama 3.2: Revolutionizing edge ai and vision with open, customizable models},
  author  = {Meta, AI},
  journal = {Meta AI Blog. Retrieved December},
  volume  = {20},
  pages   = {2024},
  year    = {2024}
}

@incollection{goodhart1984problems,
  title     = {Problems of monetary management: the UK experience},
  author    = {Goodhart, Charles AE},
  booktitle = {Monetary theory and practice: The UK experience},
  pages     = {91--121},
  year      = {1984},
  publisher = {Springer}
}

@article{strehl2008analysis,
  title   = {An analysis of model-based interval estimation for Markov decision processes},
  author  = {Strehl, Alexander L and Littman, Michael L},
  journal = {Journal of Computer and System Sciences},
  volume  = {74},
  number  = {8},
  pages   = {1309--1331},
  year    = {2008}
}

@inproceedings{bellemare2016unifying,
  title     = {Unifying count-based exploration and intrinsic motivation},
  author    = {Bellemare, Marc and Srinivasan, Sriram and Ostrovski, Georg and Schaul, Tom and Saxton, David and Munos, Remi},
  booktitle = {Advances in Neural Information Processing Systems},
  volume    = {29},
  year      = {2016}
}

@inproceedings{tang2017exploration,
  title     = {\# Exploration: A study of count-based exploration for deep reinforcement learning},
  author    = {Tang, Haoran and Houthooft, Rein and Foote, Davis and Stooke, Adam and Xi Chen, OpenAI and Duan, Yan and Schulman, John and DeTurck, Filip and Abbeel, Pieter},
  booktitle = {Advances in Neural Information Processing Systems},
  volume    = {30},
  year      = {2017}
}

@inproceedings{kolter2009near,
  title     = {Near-{Bayesian} exploration in polynomial time},
  author    = {Kolter, J Zico and Ng, Andrew Y},
  booktitle = {Proceedings of International Conference on Machine Learning},
  pages     = {513--520},
  year      = {2009}
}
\bibliographystyle{configuration/icml2026}

%%%%%%%%%%%%%%%%%%%%%%%%%%%%%%%%%%%%%%%%%%%%%%%%%%%%%%%%%%%%%%%%%%%%%%%%%%%%%%%
%%%%%%%%%%%%%%%%%%%%%%%%%%%%%%%%%%%%%%%%%%%%%%%%%%%%%%%%%%%%%%%%%%%%%%%%%%%%%%%
% APPENDIX
%%%%%%%%%%%%%%%%%%%%%%%%%%%%%%%%%%%%%%%%%%%%%%%%%%%%%%%%%%%%%%%%%%%%%%%%%%%%%%%
%%%%%%%%%%%%%%%%%%%%%%%%%%%%%%%%%%%%%%%%%%%%%%%%%%%%%%%%%%%%%%%%%%%%%%%%%%%%%%%
\newpage
\appendix
\onecolumn
\tableofcontents

\section{Theoretical Foundation}
\subsection{Count-Based Exploration Theoretical Foundation} \label{appedix: count_based_exploration}
The exploration-exploitation tradeoff is a long-standing challenge in RL literature.
  This becomes more pronounced in LLM reasoning tasks due to the vast state-action space with reward sparsity, where existing methods often suffer from deficient exploration and poor sample efficiency.
  A classic, theoretically justified exploration principle is \textit{optimism in the face of uncertainty}, which augments the reward estimates of less-visited states/actions with an exploration bonus proportional to their uncertainty.
  In the minimal multi-armed bandit (MAB) setting, the well-known upper-confidence bound (UCB) algorithm~\citep{auer2002using} chooses the action $a_t$ according to:
  \begin{equation}\label{eq:ucb}
    a_t = \mathop{\arg\max}_{a\in\mathcal{A}} \hat{R}_t(a) + \sqrt{\frac{2\log t}{n_t(a)}},
  \end{equation}
  where $\hat{R}_t(a)$ and $n_t(a)$ is the reward estimate and visitation count for action $a$ at time $t$.
  Theorem~\ref{theo:mab} establishes the theoretical optimality of this \textit{count-based exploration} strategy, demonstrating that it yields a cumulative regret that grows only logarithmically over time.
  \begin{theorem}\label{theo:mab} \textup{\citep{auer2002using}}
    In an MAB setting, let $L(T)=\mathbb{E}[\sum_{t=1}^T\left(R^*-R(a_t)\right)]$ denote the total regret over $T$ steps, where $R(a)=\mathbb{E}_{\mathcal{R}^a}[R]$ is the expected reward for any action $a$ and $R^*$ is the reward for the optimal action.
    Let $\Delta_a=R^*-R(a)$ denote the reward gap between action $a$ and the optimum.
    For any bandit algorithm, the asymptotic total regret is at least logarithmic in the number of steps:
    \begin{equation}
      \lim_{T\to\infty}L(T)\ge \log T\cdot \sum_{a|\Delta_a>0}\frac{\Delta_a}{\text{KL}(\mathcal{R}^a||\mathcal{R}^{a^*})}.
    \end{equation}
    The UCB algorithm in Eq.~\ref{eq:ucb} achieves logarithmic asymptotic total regret:
    \begin{equation}
      \lim_{T\to\infty}L(T)  \le 8\log T \cdot\sum_{a|\Delta_a>0}\Delta_a.
    \end{equation}
  \end{theorem}
  Subsequent studies~\citep{bellemare2016unifying,tang2017exploration} extend this principle to MDPs by counting state-action pairs $n(s,a)$ and adding a bonus reward to encourage exploring less-visited pairs as
  \begin{equation}\label{eq:sa_count}
    R'(s,a) = R(s,a) + \frac{\beta}{\sqrt{n(s,a)}},
  \end{equation}
  which is shown to be formally near-optimal within the probably approximately correct (PAC-MDP) framework~\citep{strehl2008analysis,kolter2009near}.

% \begin{figure}[h]
    
%   \centering
%   \includegraphics[width=\linewidth]{sources/figures/train/four_metrics_grid_10gram.pdf}
%   \caption{local count across train data }
%   \label{fig:baseline_count}

% \end{figure}

\subsection{Exponential Decay of Long Sequences} \label{appendix:lemma_proof}
Broadening the coverage of the state space naturally drives the model to discover rare yet correct reasoning chains, preventing premature convergence to local optima in sparse-reward landscapes and incentivizing reasoning capacities.
  However, naively counting $n(s,a)$ in a vast language space becomes computationally prohibitive and statistically inefficient.
  Instead, we employ the response length as a coarse-grained abstraction of the state space.
  Lemma~\ref{appendix_lemma:length} shows that the expected count of a response decays exponentially w.r.t. its length, resulting in a long-tailed distribution.
  \begin{lemma}[\textbf{Exponential Decay of Long Sequences}] \label{appendix_lemma:length}
    Let $\mathcal{S}_L$ denote the set of states (i.e., responses in LLM reasoning) with sentence length $L$, and $p(\mathcal{S}_L)$ denote the sampling probability for that state set.
    Then, there exists a constant $\epsilon\in (0,1)$, such that the sampling probability of $\mathcal{S}_L$ is upper-bounded by an exponentially decaying function of its length $L$:
    \begin{equation}
      p(\mathcal{S}_L) < (1-\epsilon)^{L-1}.
    \end{equation}
  \end{lemma}
  \begin{proof}
    Given a generative LLM model, let $p_L([\text{EOS}])$ denote the sampling probability of the end token at step $L$.
    With the softmax function as the output probability in language models, we have $p_L([\text{EOS}])\in (0,1),~\forall L$.
    There exists a lower bound of sampling probability $\epsilon\in (0,1)$, such that $p_L([\text{EOS}]) \ge \epsilon,~\forall L$.
    In an autoregressive manner, the sampling probability of a response with length $L$ is:
    \begin{equation}\label{eq:sample_prob}
      \begin{aligned}
        p(\mathcal{S}_L) & = \prod_{i=1}^{L-1}\left(1-p_i([\text{EOS}])\right)\cdot p_L([\text{EOS}]) \\
                         & \le \left(1-\epsilon\right)^{L-1}\cdot p_L([\text{EOS}])                   \\
                         & < (1-\epsilon)^{L-1}.
      \end{aligned}
    \end{equation}
    This completes the proof.
  \end{proof}

\section{Experimental Details}
\subsection{Template Prompt}
We adopt the following template for all experiments involving Qwen models, building upon the Qwen3 template for Qwen3-4B-Base and the Qwen-Nothinking template for Qwen3.

\begin{tcolorbox}[
    title={Qwen3 Template}
]
<|im\_start|>user\\
\{problem\} Let's think step by step and output the final answer within \textbackslash boxed\{\}.\\
<|im\_end|>\\
<|im\_start|>assistant
\end{tcolorbox}

\begin{tcolorbox}[
    title={Qwen3-NoThinking Template}
]
<|im\_start|>user\\
\{problem\} Let's think step by step and output the final answer within \textbackslash boxed\{\}.\\
<|im\_end|>\\
<|im\_start|>assistant\\
<think>\\
\\
</think>
\end{tcolorbox}
For training the Llama-OctoThinker models, we adopt the original prompt in \cite{wang2025octothinker} to ensure performance.
\begin{tcolorbox}[
    title={OctoThinker Template}
]
A conversation between User and Assistant. The user asks a question, and the Assistant solves it. The assistant first thinks about the reasoning process in the mind and then provides the user with the answer. User: You must put your answer inside \textbackslash boxed\{\} and Your final answer will be extracted automatically by the \textbackslash boxed\{\} tag. \\
\{problem\} \\
Assistant:
\end{tcolorbox}

\subsection{Implementation of RL}
In this section, we describe the RL training setup in detail.
We implement GRPO and other baseline algorithms using the Verl framework.
Across all algorithms and model variants, we adopt a unified set of hyperparameters, as reported in Table \ref{tab:rl_hyperparameters}, and do not employ entropy regularization or KL-based losses.
\begin{table}[h]
    \centering
    \caption{RL Hyperparameters.}
    \label{tab:rl_hyperparameters}
    \vspace{2mm}
    \begin{tabular}{ll}
        \toprule
        \textbf{Hyperparameter} & \textbf{Value} \\
        \midrule
        Optimizer & AdamW \\
        Policy learning rate & $1\text{e}^{-6}$ \\
        Training batch size & 128 prompts \\
        Samples per prompt & 8 \\
        Mini-batch size & 32 prompts \\
        Policy updates per rollout & 16 \\
        Max prompt length & 1024 tokens \\
        Max response length & 8192 \\
        Rollout temperature & 1.0 \\
        \bottomrule
    \end{tabular}
\end{table}

\section{Additional Results}
\subsection{Pitfalls of Direct State Coverage Maximization} \label{appendix:distinct_reward}
As discussed in Section ~\ref{sec:problem}, a natural inclination for incentivizing in-context exploration is to directly use the In-Context Distinct State Count $C_{\text{context}}(\tau)$ as a reward signal.
However, our empirical analysis in Figure~\ref{fig:bonus_ablate} reveals critical failure modes: reward hacking.

When $C_{\text{context}}$ is directly maximized (as shown in the red curves), the model initially exhibits a sharp increase in state coverage.
However, this is quickly followed by a "collapse" where the model learns to generate semantically hollow but structurally diverse tokens (e.g., random character permutations or irrelevant LaTeX symbols) to artificially inflate the state count.
This results in a significant drop in reasoning accuracy.
In contrast, our LIE recipe (purple curves) uses length as a proxy for capacity and redundancy as a constraint, which leads to a more stable and meaningful expansion of the state space, ultimately translating into higher and more stable task performance.
\begin{figure}[h]
  \centering
  \includegraphics[width=\linewidth]{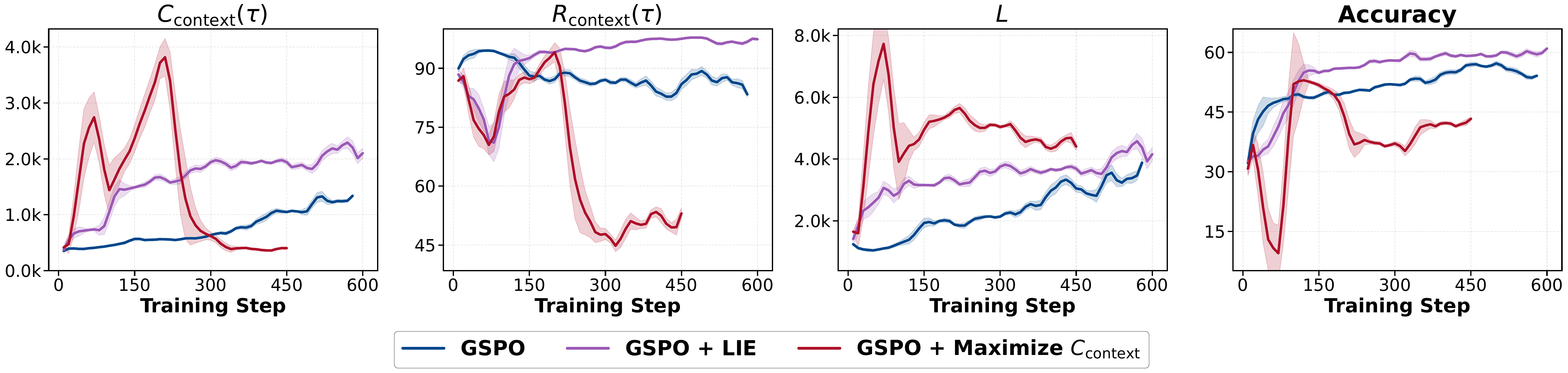}
  \caption{Training dynamics of maximizing $C_{\text{distinct}}$ as a reward.}
  \label{fig:bonus_ablate}
  \vspace{-1.2em}
\end{figure}

\subsection{Training dynamics of OctoThinker}
As discussed in Section ~\ref{results:different_models}, our Length-Incentivized (LIE) recipe demonstrates strong generalization across various model architectures.
Figure~\ref{fig:llama} illustrates the training dynamics of OctoThinker-3B-Base, comparing standard GSPO with our \method-enhanced approach. 
These observations reinforce that "thinking longer" via length incentivization is a reliable mechanism for performance extrapolation, regardless of the underlying policy model.
\begin{figure*}[h]
    \centering
    \includegraphics[width=\linewidth]{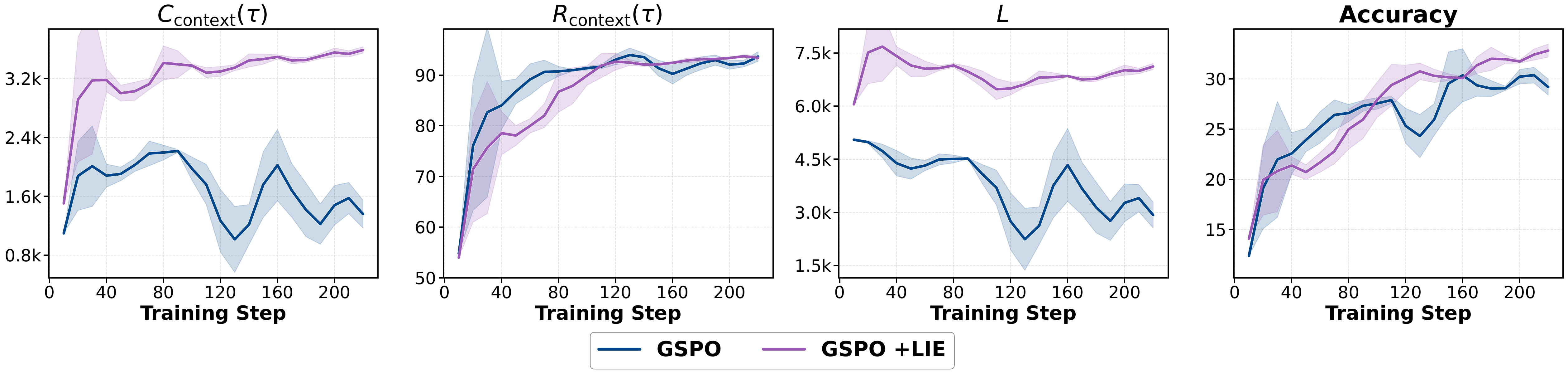}
    \caption{The training dynamics of OctoThinker-3B-Base.}
    \label{fig:llama}
\end{figure*}

\subsection{Scalability Across Model Scales} \label{appendix:different_model_size}
To verify the scalability of \method, we evaluate the recipe across different model sizes: Qwen3-1.7B-Base, 4B-Base, and 8B-Base.
As summarized in Table~\ref{tab:different_model_size}, \method consistently delivers performance gains regardless of the base model's capacity. 

Specifically, on the in-domain reasoning average, \method improves the 1.7B model by 7.6\%, the 4B model by 4.4\%, and the 8B model by 2.5\% compared to the GSPO baseline.
Notably, on out-of-distribution (OOD) tasks, the 8B model with LIE achieves a 73.4\% average accuracy, a 2.9\% improvement over the baseline.
These results demonstrate that the "Shallow Exploration Trap" is a universal bottleneck for LLM reasoning, and \method provides a robust mechanism to unlock deeper reasoning capabilities across the scaling spectrum.
\begin{table}[h]
  \caption{\method on different model sizes.}
  \label{tab:different_model_size}
  \resizebox{\columnwidth}{!}{%
  \begin{tabular}{ccccccc|cccc}
  \toprule
  \multirow{2}{*}{\textbf{Model}} & \multicolumn{6}{c|}{\textbf{In-Domain Performance}}                                                 & \multicolumn{4}{c}{\textbf{Out-of-Domain Performance}}             \\ 
                                  & \textbf{MATH} & \textbf{Olympiad} & \textbf{AMC}  & \textbf{AIME} & \textbf{AIME25} & \textbf{Avg.} & \textbf{ARC-c} & \textbf{GPQA} & \textbf{MMLU-Pro} & \textbf{Avg.} \\ \midrule
  \textbf{Qwen3-1.7B-Base}        & 51.2          & 20.7              & 25.8          & 3.4           & 1.7             & 20.6          & 54.1           & 20.2          & 27.5              & 33.9          \\
  \textbf{GSPO}                   & 70.6          & 32.3              & 38.4          & 8.1           & 3.8             & 30.6          & \textbf{79.4}  & 28.2          & 41.0              & 49.5          \\
  \textbf{+ \method}            & \textbf{77.0} & \textbf{39.6}     & \textbf{45.2} & \textbf{17.5} & \textbf{11.5}   & $\mathbf{38.2}_{\textcolor{blue}{+7.6}}$ & 79.1           & \textbf{33.8} & \textbf{45.6}     & $\mathbf{52.8}_{\textcolor{blue}{+3.3}}$ \\ \midrule
  \textbf{Qwen3-4B-Base}          & 66.0          & 33.2              & 36.6          & 8.5           & 6.9             & 30.2          & 66.9           & 26.3          & 30.9              & 41.4          \\
  \textbf{GSPO}                   & 85.2          & 51.7              & 62.7          & 26.7          & 20.5            & 49.4          & 88.4           & \textbf{48.5} & 61.5              & 66.1          \\
  \textbf{+ \method}            & \textbf{88.4} & \textbf{57.2}     & \textbf{66.2} & \textbf{30.5} & \textbf{26.7}   & $\mathbf{53.8}_{\textcolor{blue}{+4.4}}$ & \textbf{91.4}  & 47.5          & \textbf{63.8}     & $\mathbf{67.6}_{\textcolor{blue}{+1.5}}$ \\ \midrule
  \textbf{Qwen3-8B-Base}          & 67.8          & 35.3              & 38.9          & 10.3          & 8.5             & 32.2          & 58.5           & 32.3          & 51.2              & 47.3          \\
  \textbf{GSPO}                   & 89.8          & 58.8              & 72.9          & 34.4          & 25.8            & 56.3          & 93.2           & 50.0          & 68.3              & 70.5          \\
  \textbf{+ \method}            & \textbf{91.4} & \textbf{60.4}     & \textbf{73.4} & \textbf{37.2} & \textbf{31.6}   & $\mathbf{58.8}_{\textcolor{blue}{+2.5}}$ & \textbf{94.5}  & \textbf{55.1} & \textbf{70.7}     & $\mathbf{73.4}_{\textcolor{blue}{+2.9}}$ \\ \bottomrule
  \end{tabular}%
  }
  \end{table}

\subsection{Synergizing with SFT: Injection vs. Activation}
We further investigate the interplay between Supervised Fine-Tuning (SFT) and our \method.
We posit that these two stages are \textbf{orthogonal and complementary}: SFT serves as the \textit{injection} stage, embedding specific reasoning patterns and primitives into the model, while LIE serves as the \textit{activation} mechanism, compelling the model to compose these patterns into longer, more complex reasoning chains.

\textbf{Setup.} We constructed a curated SFT dataset by randomly sampling 4k problems from the training set and generating oracle Chain-of-Thought (CoT) data using GPT-OSS-120B~\citep{agarwal2025gpt}.
We fine-tuned the Qwen3-4B-Base model on this data for 3 epochs with a 4,096 token context limit.

\textbf{Results.}  The training dynamics in Figure~\ref{fig:sft} strongly support our hypothesis:

\begin{enumerate}
    \item \textbf{SFT provides the capability foundation (Injection).} Both SFT-initialized runs (Red and Purple) start with significantly higher state coverage ($C_{\text{context}}$) and accuracy compared to the Base model (Blue), confirming that SFT successfully injects necessary reasoning primitives.
    \item \textbf{\method drives further extrapolation (Activation).} While the standard SFT+GSPO baseline (Purple) shows improvement, it tends to plateau in trajectory length ($L \approx 5k$) and coverage, remaining tethered to the SFT data distribution. In contrast, applying LIE on top of SFT (Red) effectively "activates" the model to explore beyond the supervised horizon. The model utilizes the injected patterns to sustain a continued expansion in reasoning length ($L \rightarrow 7k+$) and distinct state coverage. 
    \item \textbf{Orthogonality.} The highest performance is achieved by \texttt{SFT GSPO + ILE} (Red). This demonstrates that our recipe functions as an independent scaling lever that synergizes with the strong priors from SFT to unlock the maximum reasoning potential.
\end{enumerate}
\begin{figure}[H]
\centering
\includegraphics[width=\linewidth]{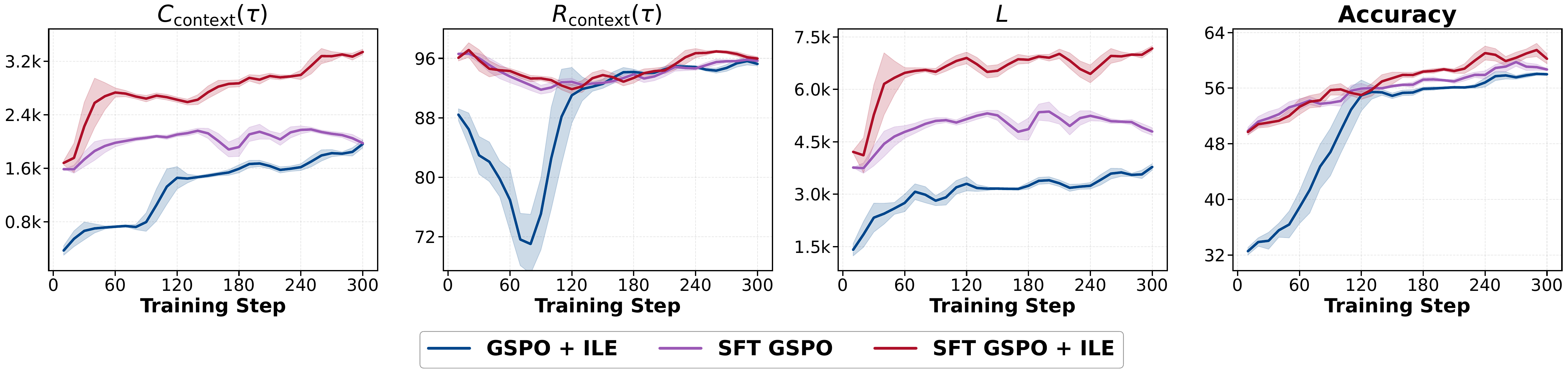}
\caption{Training Dynamics on Qwen3-4B-Base after SFT.}
\label{fig:sft}
\vspace{-1.2em}
\end{figure}

\subsection{Hyperparameter Sensitivity Analysis} \label{appendix:ablation}
We conduct an ablation study to verify that our method's improvements are robust to hyperparameter choices and to understand the impact of exploration constraints.

\paragraph{Sensitivity to State Abstraction ($n$-gram).} 
\begin{figure}[H]
\centering
\includegraphics[width=\linewidth]{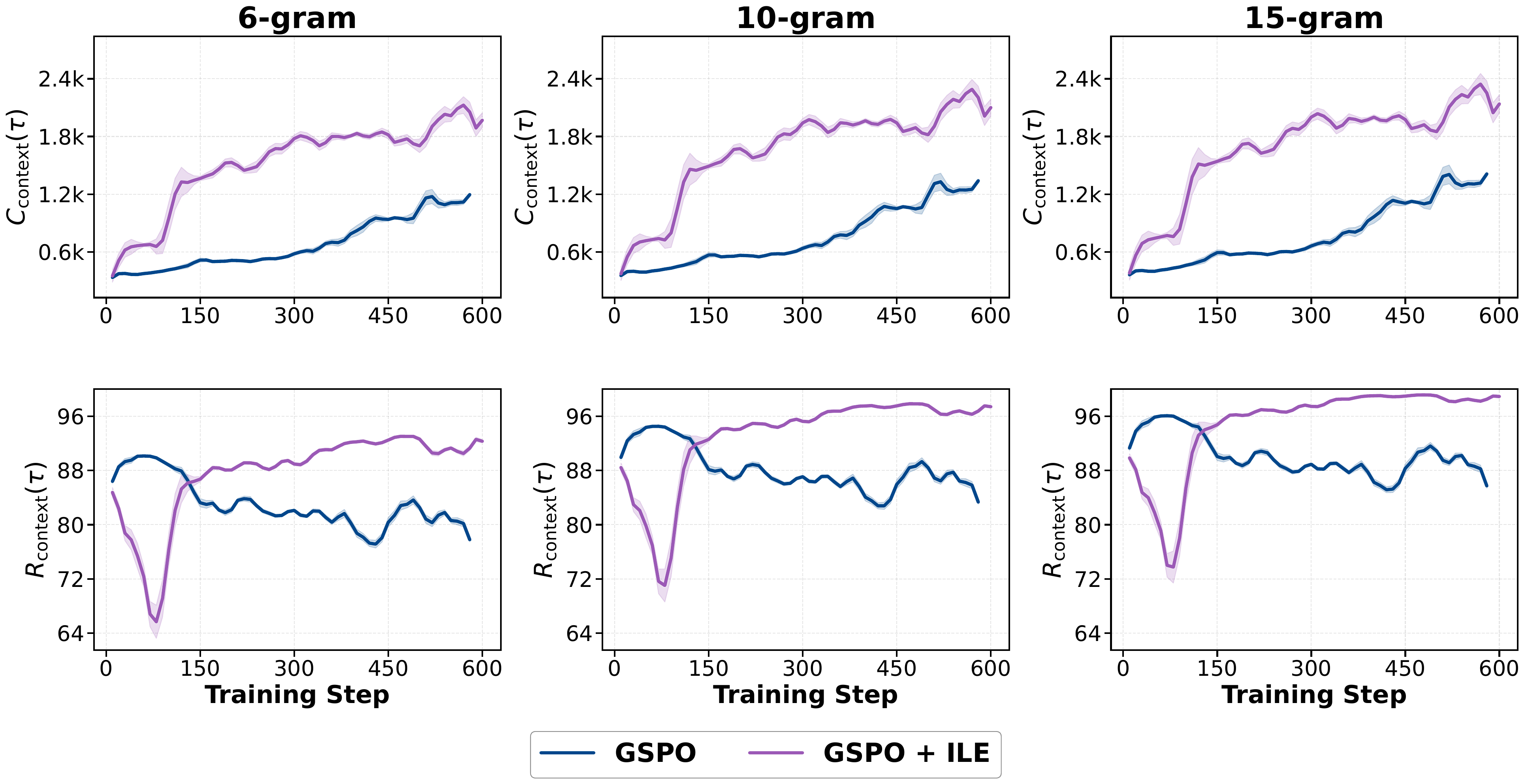}
\caption{$C_{\text{context}}$ and $R_{\text{context}}$ dynamics across {6, 10, 15}-grams.}
\label{fig:ngram_sensitive}
\vspace{-1.2em}
\end{figure}
We first analyze the impact of the context window size $n$ used for defining state abstraction (Equation~\ref{eq:state_abstraction}).
As shown in Figure~\ref{fig:ngram_sensitive}, the relative performance trends between standard GSPO and our LIE-enhanced method remain consistent across $n \in \{6, 10, 15\}$, confirming that the improvements are not an artifact of a specific metric configuration. 
However, qualitative inspection reveals the trade-offs at extreme values:
\begin{itemize}
    \item \textbf{Small $n$ ($n=6$):} The metric becomes overly sensitive to trivial syntactic patterns.
    We observed that $n=6$ frequently flags structural tokens (e.g., repeated LaTeX formatting like ``\texttt{\textbackslash \ \textbackslash}'') as redundancy, introducing noise into the reward signal.
    \item \textbf{Large $n$ ($n=15$):} The state space becomes overly sparse. With longer windows, the \textit{Unique State Count} remains naturally high even for semantically repetitive chains, making it harder to detect the "Shallow Exploration Trap."
\end{itemize}
We adopt $n=10$ as the optimal balance, capturing semantic logical steps while filtering out low-level syntactic noise.

\paragraph{Impact of Lengthening ($\Delta L$).}
In the training stage, we further ablate the effect of the exploration increment limit $\Delta L$, which governs the pace at which we force the model to expand its reasoning horizon. As illustrated in Figure~\ref{fig:delta_l_ablate}, we observe that:
\begin{itemize}
    \item \textbf{Moderate Extension ($\Delta L = 100, 500$):} These settings successfully activate the "capacity for exploration," yielding stable growth in both trajectory length and accuracy.
    \item \textbf{Aggressive Extension ($\Delta L = 8k$):} Removing the progressive constraint (i.e., immediately demanding max length) destabilizes training. The model, unable to legitimately reason at such lengths instantly, collapses into degenerate repetition loops to satisfy the length reward, leading to a degradation in reasoning quality.
\end{itemize}
\begin{figure*}[]
    \centering
    \includegraphics[width=\linewidth]{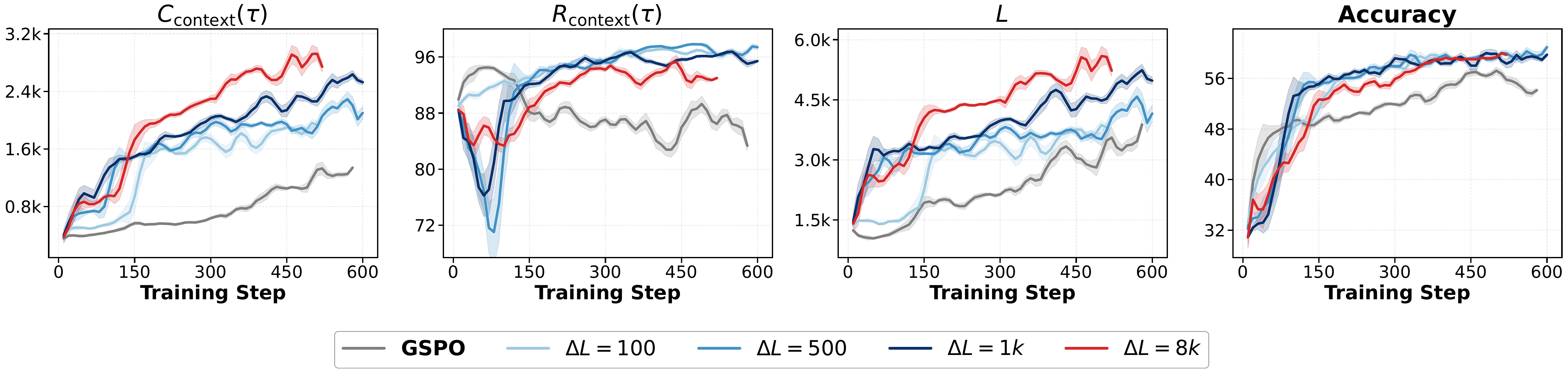}
    \caption{Training dynamics of different $\Delta L$}
    \label{fig:delta_l_ablate}
    % \vspace{-1.2em}
\end{figure*}

\paragraph{Impact of Repetition Threshold ($\Theta$).}
We examine the sensitivity of the redundancy penalty threshold $\Theta$. As shown in Table~\ref{tab:ablate_theta}, our method achieves optimal performance at $\Theta=10$ (Avg 52.7\%), surpassing both stricter ($\Theta=6$) and looser ($\Theta=15$) settings. Notably, the strict threshold ($\Theta=6$) leads to a performance drop on the challenging AIME25 benchmark (21.5\%) compared to the baseline (22.2\%).
We speculate the reasons are twofold:
\begin{itemize}
    \item \textbf{Over-penalization ($\Theta=6$):} Mathematical reasoning requires a degree of natural repetition (e.g., restating variables or structural connectors like "therefore"). An overly strict threshold likely triggers false positives, penalizing valid reasoning steps and disrupting the logical flow required for complex problems.
    \item \textbf{Ineffective Constraint ($\Theta=15$):} A loose threshold fails to sufficiently curb the "Shallow Exploration Trap." It allows the model to "game" the length-incentivized reward ($R_{\text{len}}$) by generating repetitive low-entropy sequences that satisfy the length requirement without achieving genuine state expansion, thereby reducing exploration efficiency.
\end{itemize}

\begin{table}[H]
\centering
\caption{The Ablation study of $\Theta$.}
\label{tab:ablate_theta}
\resizebox{0.7\columnwidth}{!}{%
\begin{tabular}{ccccccc}
\toprule
\multicolumn{1}{l}{\textbf{}} & \textbf{MATH} & \textbf{Olympiad} & \textbf{AIME} & \textbf{AMC}  & \textbf{AIME25} & \textbf{Avg.} \\ \midrule
\textbf{GRPO w/Clip-Higher}   & 86.4          & 54.1              & 25.2          & 61.8          & 22.2            & 49.9          \\ \midrule
\textbf{+\method($\Theta=6$)}           & 86.0          & 55.4     & 29.2          & 65.4          & 21.5            & 51.5          \\
\textbf{+\method($\Theta=10$)}          & \textbf{88.8} & 54.1              & \textbf{30.4} & \textbf{65.5} & \textbf{24.5}   & \textbf{52.7} \\
\textbf{+\method ($\Theta=15$)}          & 86.8          & \textbf{56.0}              & 28.5          & 63.9          & 22.4            & 51.5          \\ \bottomrule
\end{tabular}%
}
\end{table}

\paragraph{Impact of Repetition Magnitude $\beta$}
The redundancy penalty weight $\beta$ (Equation~\ref{eq:rep_penalty}) serves as a hyperparameter for balancing exploration capacity and efficiency.
As illustrated in the training dynamics in Figure~\ref{fig:beta_ablate}, we observe that the final reasoning accuracy is robust to the choice of $\beta$.
Both $\beta=0.3$ and $\beta=0.6$ achieve nearly identical convergence in accuracy, indicating that the penalty magnitude does not affect the model's fundamental problem-solving capability.

The primary effect of a higher $\beta$ is the promotion of conciseness. As shown in the length ($L$) and entropy plots, $\beta=0.6$ leads to shorter reasoning trajectories and a more rapid decrease in policy entropy compared to $\beta=0.3$.
This suggests that while a smaller $\beta$ allows for more ``loose'' exploration, a larger $\beta$ encourages the model to find more efficient, high-density reasoning paths to the correct answer.
\begin{figure}[h]
    \centering
    \includegraphics[width=\linewidth]{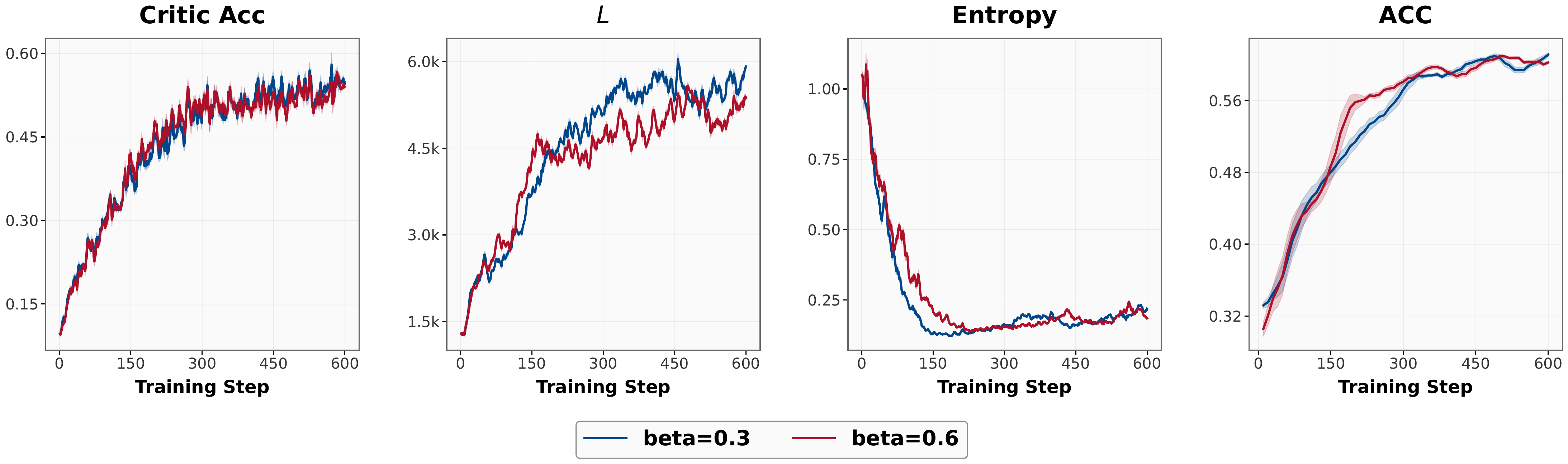}
    \caption{Training dynamics of different $\beta$.}
    \label{fig:beta_ablate}
    \vspace{-1.2em}
\end{figure}

\subsection{Ablation Study of the Reward Shaping Recipe} \label{appendix:reward_ablation}
To understand the contribution of each component in our final recipe (Equation~\ref{eq:reward_recipe}), we perform an exhaustive ablation study.
As illustrated in Figure~\ref{fig:reward_ablate}, the interaction between the length reward ($R_{\text{len}}$) and the redundancy penalty ($R_{\text{red}}$) is key to unlocking effective in-context exploration.
\begin{itemize}
    \item \textbf{Length reward ($R_{\text{len}}$) as the capacity driver:} (Proposition~\ref{propos:length_as_capacity}) The brown curves (GSPO + $R_{\text{len}}$) demonstrate that explicitly incentivizing length successfully breaks the ``Shallow Exploration Trap,'' leading to a surge in the total distinct state count $C_{\text{context}}$. However, this comes at a significant cost: the distinct ratio $R_{\text{context}}$ plummets, and length becomes extremely large. This indicates that without constraints, the model tends to satisfy the length requirement through ``thought padding'' --- generating repetitive or low-value tokens.
    
    \item \textbf{Redundancy reward ($R_{\text{red}}$) as the quality filter:} Conversely, when only the redundancy penalty is added (red curves), the model maintains a high $R_{\text{context}}$, but the average trajectory length $L$ remains low. This suggests that while the penalty prevents repetition, it does not provide the necessary drive to navigate deeper into the state space for complex reasoning.
    
    \item \textbf{Synergistic effect of the \method recipe:} Our complete recipe (GSPO + LIE, purple curves) achieves a superior balance. By combining $R_{\text{len}}$ and $R_{\text{red}}$, we create a synergistic effect: $R_{\text{len}}$ provides the \textit{computational budget} to explore, while $R_{\text{red}}$ ensures that this budget is spent on \textit{diverse and meaningful} reasoning states. As a result, the model exhibits a steady growth in both $C_{\text{context}}$ and length, while maintaining a healthy $R_{\text{context}}$. This structured exploration ultimately leads to the highest and most stable accuracy gains among all variants.
\end{itemize}

In summary, the \method recipe is not merely a sum of its parts; it is a mechanism that converts raw sequence length into effective reasoning depth.

\begin{figure}[h]
  \centering
  \includegraphics[width=\linewidth]{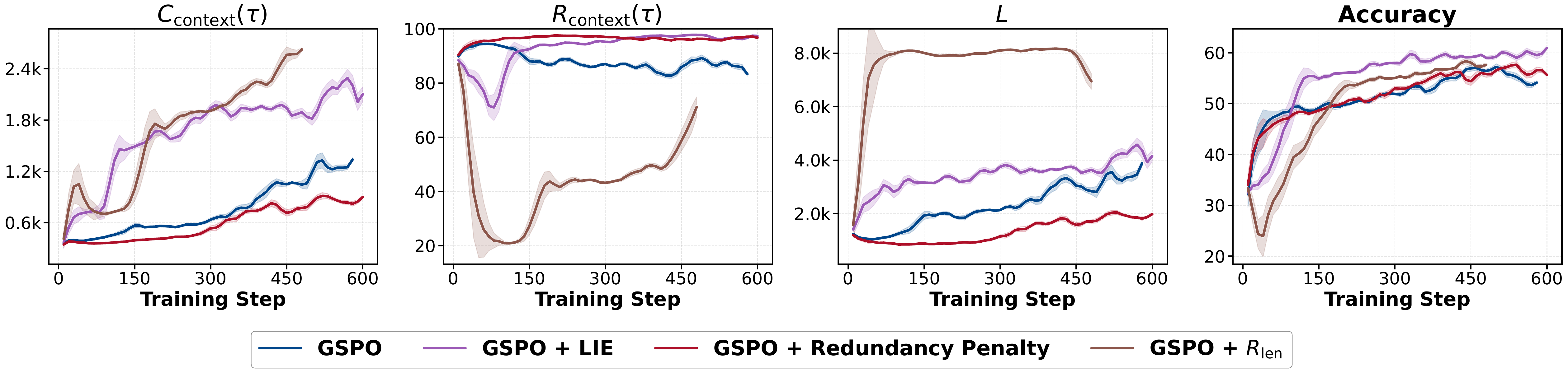}
  \caption{Training dynamics of different reward components.}
  \label{fig:reward_ablate}
  \vspace{-1.2em}
\end{figure}

% \begin{table}[]
% \centering
% \caption{Average response length during test-time.}
% \resizebox{0.7\columnwidth}{!}{%
% \begin{tabular}{ccccccc}
% \toprule
% \textbf{Model}                          & \textbf{MATH} & \textbf{Olympiad} & \textbf{AMC} & \textbf{AIME} & \textbf{AIME25} & \textbf{Avg.} \\ \midrule
% \textbf{GRPO}                           & 1173.9        & 2394.1            & 2482.2       & 4470          & 3054.1          & 2714.9        \\
% \textbf{GRPO + \method}                 & 1686.4        & 2729              & 3173.3       & 4685.9        & 4077.2          & 3270.4        \\
% \bottomrule
% \end{tabular}%
% }
% \end{table}

\section{Case Study} \label{appendix:case_study}
To analyze the internal reasoning dynamics, we follow the principle of \cite{feng2025characterizes} and employ Claude-3.7-Thinking to extract the underlying reasoning graph structure from the model's generated responses.
This allows us to visualize and quantify the "thought process" beyond simple token sequences or behaviors.

\paragraph{Qualitative Analysis.}
We select a representative problem from the AIME benchmark where the baseline fails while our method succeeds. As visualized in Figure~\ref{fig:gspo_case_study}, the \textbf{GSPO baseline} (Left) follows a linear and shallow reasoning path.
It attempts a direct derivation but fails to cross-check its intermediate steps, quickly converging to an incorrect answer ($324$).

In contrast, the model trained with our \method recipe (Figure~\ref{fig:LIE_case_study}) exhibits a significantly richer reasoning topology.
Crucially, the length incentive activates \textit{in-context exploration} behaviors: the model spawns alternative hypotheses (``Alternative Approach: Using Polar Form''), performs explicit self-verification (``Step 5: Verification''), and successfully identifies and corrects a calculation error (``Calculation Error'').
This capacity to branch out and backtrack allows the model to recover from initial pitfalls and ultimately derive the correct solution ($540$).

\paragraph{Quantitative Structural Metrics}
To verify if this observation holds statistically, we compute the average \textbf{Depth} (maximum path length in the reasoning graph) and \textbf{Width} (average branching factor) across the 40 samples from AIME.
As shown in Table ~\ref{tab:structure}, our method consistently expands the reasoning structure:

\begin{itemize}
    \item \textbf{Increased Depth ($13.8 \rightarrow 14.75$):} The model constructs longer logical chains, enabling deeper decomposition of complex problems.
    \item \textbf{Increased Width ($2.15 \rightarrow 2.30$):} More importantly, the increased width indicates that the model is not merely ``padding'' the response with empty tokens. Instead, it engages in broader exploration by considering multiple parallel hypotheses or verification paths.
\end{itemize}

These structural metrics confirm that explicitly incentivizing response length effectively translates into a broader search horizon, enabling the model to navigate the state space more thoroughly.

\begin{figure}[htbp]
    \centering
    \begin{subfigure}[b]{0.3\textwidth}
      \centering
      \includegraphics[width=\textwidth]{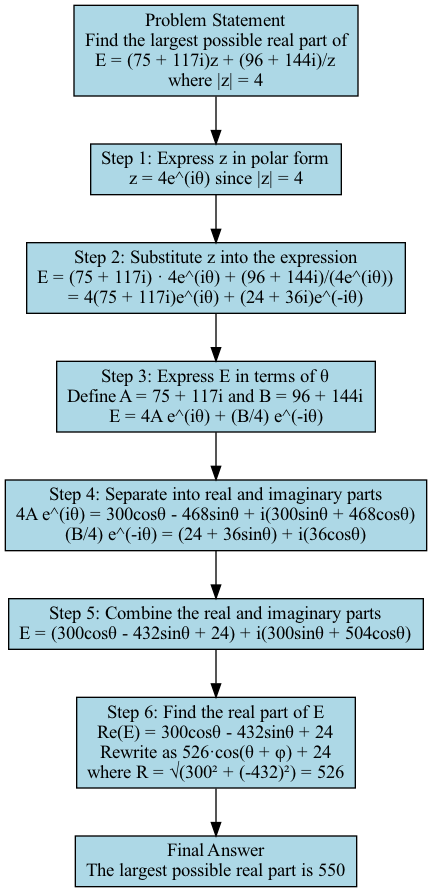}
      \caption{GSPO}
      \label{fig:gspo_case_study}
    \end{subfigure}
    \hfill
    \begin{subfigure}[b]{0.5\textwidth}
      \centering
      \includegraphics[width=\textwidth]{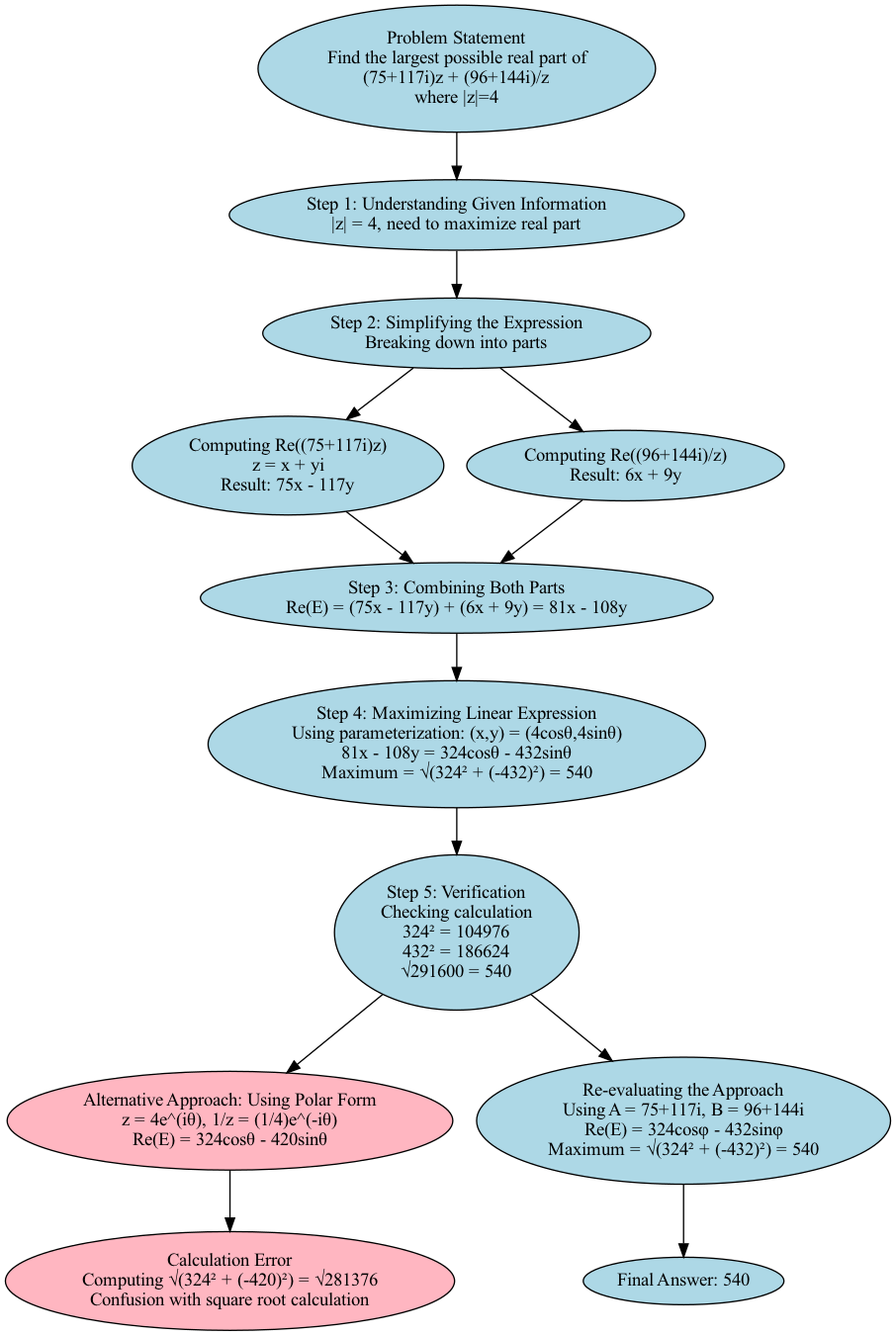}
      \caption{GSPO + \method}
      \label{fig:LIE_case_study}
    \end{subfigure}
    \caption{A case from AIME on GSPO and our recipe models.}
    \label{fig:case_study}
\end{figure}

\begin{table}[h]
\caption{Exploration width and depth metrics.}
  \label{tab:structure}
  \centering
  \resizebox{0.3\columnwidth}{!}{%
  \begin{tabular}{ccc}
  \toprule
  \textbf{Model}          & \textbf{Depth} & \textbf{Width} \\ \midrule
  \textbf{GSPO}           & 13.8           & 2.15           \\
  \textbf{GSPO + \method} & 14.75          & 2.3            \\ \bottomrule
  \end{tabular}%
}
\end{table}

\end{document}